\documentclass[format=acmsmall, review=true]{acmart}
\usepackage{acm-ec-21}
\usepackage{booktabs} %
\usepackage[ruled]{algorithm2e} %

\SetAlFnt{\small}
\SetAlCapFnt{\small}
\SetAlCapNameFnt{\small}
\SetAlCapHSkip{0pt}
\IncMargin{-\parindent}

\usepackage{amsmath}
\usepackage{bbm}
\usepackage{amsthm}
\usepackage{hyperref}

\DeclareMathOperator*{\argmax}{arg\,max}

\DeclareMathOperator*{\Var}{Var}
\newcommand{\reals}{\mathbb{R}}
\newcommand{\E}{\mathop{\mathbb{E}}}
\newcommand{\X}{\mathcal{X}}
\renewcommand{\H}{\mathcal{H}}

\newcommand{\empiricalquadscore}{Q}
\newcommand{\expectedquadscore}{S}
\newcommand{\expectedtruescore}{S^*}

\newcommand{\elf}{M_{\mathrm{ELF}}}

\newcommand{\Reg}{\mathrm{Reg}^*}
\newcommand{\R}{\mathcal{R}}

\newcommand{\uco}{c}  %
\newcommand{\strat}{\sigma}  %

\newcommand\numberthis{\addtocounter{equation}{1}\tag{\theequation}}
\newtheorem{theorem}{Theorem}
\newtheorem{corollary}{Corollary}
\newtheorem{condition}{Condition}
\newtheorem{lemma}{Lemma}
\newtheorem{definition}{Definition}
\newtheorem{claim}{Claim}
\newtheorem{prop}{Proposition}

\begin{abstract}
Winner-take-all competitions in forecasting and machine-learning suffer from distorted incentives.
\citet{witkowski2101incentive} identified this problem and proposed ELF, a truthful mechanism to select a winner.
We show that, from a pool of $n$ forecasters, ELF requires $\Theta(n\log n)$ events or test data points to select a near-optimal forecaster with high probability.
We then show that standard online learning algorithms select an $\epsilon$-optimal forecaster using only $O(\log(n) / \epsilon^2)$ events, by way of a strong approximate-truthfulness guarantee.
This bound matches the best possible even in the nonstrategic setting.
We then apply these mechanisms to obtain the first no-regret guarantee for non-myopic strategic experts.
\end{abstract}

\begin{document}

\title{Efficient Competitions and Online Learning with  Strategic Forecasters}

\author{Rafael Frongillo, Robert Gomez, Anish Thilagar, Bo Waggoner}
\date{\today}

\begin{titlepage}

\maketitle

\end{titlepage}

\section{Introduction}
In forecasting competitions, a winner is selected among a pool of contestants based on the accuracy of their predictions.
Examples include the Good Judgement Project~\cite{gjp}, ProbabilitySports~\cite{probsports}, the Hybrid Forecasting Competition~\cite{hybridforecasting}, and also machine-learning competitions such as Kaggle~\cite{kaggle}.
The typical mechanism used is quite simple: tally the empirical accuracy of each forecaster (or machine-learning model, as the case may be) and select the highest.
Yet, the winner-take-all nature of this mechanism can skew incentives dramatically: forecasters may misreport to increase the variance of their score, even at the cost of its expectation, to improve their chance of being selected~\cite{lichtendahl2007probability,witkowski2018incentive,witkowski2101incentive,aldous2019prediction}.
Moreover, this incentive problem is not just theoretical: the winner of Kaggle's March Mania 2017 competition admitted to skewing their ``true'' predictions~\cite{kaggle2017march}.

To address this incentive problem \citet{witkowski2101incentive, witkowski2018incentive} propose the Event Lotteries Forecast (ELF) mechanism.
For binary events (or labels), which we also study in this paper, ELF probabilistically awards one point per event and selects the forecaster with the most points at the end.
The authors prove that ELF is dominant-strategy truthful and that it chooses an $\epsilon$-optimal forecaster among $n$ contestants when there are at least $O(n^2 \log n/\epsilon^2)$ events.
Unfortunately, this bound can be prohibitively large: with $n=100$ and $\epsilon=0.1$, we need on the order of 1 million events or test data points.
(For comparison, Kaggle's March Mania 2017 competition had $n=442$ forecasters, and the equivalent of just 2278 binary events.)
The main open question has therefore been whether a mechanism can select an $\epsilon$-optimal forecaster using far fewer events---perhaps ELF with a tighter analysis, or some entirely new mechanism.

We answer both parts of this question affirmatively.
First, we show that ELF needs only $O(n \log n / \epsilon^2)$ events, a quadratic improvement, and show the dependence on $n$ to be tight (\S~\ref{sec:elf}).
In the nonstrategic setting, however, only $\Theta(\log(n) / \epsilon^2)$ events are required (\S~\ref{subsec:nonstrat-lower}), raising the question of whether a new mechanism can achieve this bound with strategic forecasters.
Indeed one can: we give a new $O(\log n/\epsilon^2)$-event mechanism (in \S~\ref{sec:main-accuracy}) based on Follow the Regularized Leader (FTRL), a common class of online machine learning algorithms.
A crucial ingredient of our analysis is that, under some curvature assumptions on the regularizer, FTRL satisfies a strong notion of $\gamma$-approximate truthfulness: reports more than $\gamma$-far from one's belief are strictly dominated (\S~\ref{sec:approx-truthful}). 
While our results could pave the way to exactly-truthful mechanisms, we instead focus on our more resilient solution concept: our guarantees hold as long as forecasters play undominated strategies.

Our work also has implications for online learning from strategic experts, as studied by \citet{roughgarden2017online} and \citet{freeman2020no-regret}.
Each round, the experts make forecasts; the learning algorithm chooses an expert based on the prior rounds; and then the outcome of the round's event is revealed.
Experts are strategic, and seek (roughly) to maximize their probability of being chosen by the algorithm.
Despite this strategic behavior, we would like the algorithm's chosen forecast reports to have vanishing regret \emph{relative to the internal beliefs} of the experts.
\citet{freeman2020no-regret} give a mechanism that is no-regret and truthful for \emph{myopic} agents, which only maximize their chance of being chosen on the subsequent round and do not strategize otherwise.
The main open question has been to give a no-regret learning algorithm in the presence of general strategic experts.

Using our solution concept of undominated strategies, we show that FTRL achieves this goal (\S~\ref{sec:no-regret}).
Specifically, if forecasters wish to maximize any positive linear combination of their chances of being selected in the rounds, and they play undominated strategies, then FTRL achieves $O(\sqrt{T})$ regret after $T$ rounds.
We emphasize that the regret is to the \emph{beliefs} of the optimal forecaster, even though they may misreport; approximate truthfulness ensures that the reports are accurate enough.

The broad takeaway from our results is that, in both of these settings, perhaps surprisingly, there is no price of strategic behavior under the solution concept of undominated strategies.
For forecasting and machine learning competitions, our results suggest a benefit of relatively small changes in competition protocols.
For online learning, moreover, popular learning algorithms are essentially already robust to the type of strategic behavior we consider.

\section{Model and Preliminaries}
\label{sec:model-preliminaries}
There are $m$ independent binary events indexed by $t$, each associated with an independent random variable $y_t \in \{0,1\}$.
We assume there is an unknown ``ground truth'' probability of event $t$ occurring, $\theta_t = \Pr[y_t = 1]$.
We write $\vec{\theta} = (\theta_1,\dots,\theta_m)$ and $\vec{y} = (y_1,\dots,y_m)$.

There are $n \geq 2$ forecasters, indexed by $i$ and $j$.
On each event $t$, forecaster $i$ has an immutable belief $p_{it} \in [0,1]$ of the probability of event $t$.
Forecasters believe all events are independent.
Meanwhile, $r_{it} \in [0,1]$ will denote $i$'s reported probability of event $t$.
Let $P \in [0,1]^{n \times m}$ and $R \in [0,1]^{n \times m}$ be the matrices of all beliefs and reports, respectively.
We write $p_i = (p_{i1},\dots,p_{im})$ for the row consisting of $i$'s beliefs, and similarly for $r_i$.

\subsection{Mechanisms and truthfulness}
\label{sec:mech-truthful}
A mechanism will first solicit reports $R$, then observe outcomes $\vec{y}$, and select exactly one of the $n$ forecasters as the winner.
We write $\Delta_n$ for the probability simplex over $\{1,\ldots,n\}$.

\begin{definition}
  \label{mechanism}
   A \textbf{forecasting competition mechanism} $M$ is a family of functions $M_{n, m}:[0,1]^{n\times m}\times\{0,1\}^m \to \Delta_n$, for all $n,m \in \mathbb{N}$, where $M_{n, m}(R,\vec{y})_i$ is the probability with which the mechanism picks forecaster $i$ on reports $R$ and observed outcomes $\vec{y}$.
   As $R$ determines $n$ and $m$, we suppress the subscripts.
  For a belief $p_i \in [0,1]^m$, we write $M(R;p_i) := \E_{\vec{y} \sim p_i}[ M(R,\vec{y})]$.
\end{definition}

The utility of forecaster $i$ is simply $1$ if they are selected, $0$ otherwise.
Therefore, we define their expected utility over the randomness of the mechanism as $M(R,\vec{y})_i$; and their expected utility over the randomness of the events as well, according to their internal beliefs, is $M(R;p_i)_i$.
We write $M(\hat r_i, R_{-i}, \vec{y})$ to denote running the mechanism with $i$'s report replaced by some vector $\hat r_i \in [0,1]^m$.

\begin{definition}
  \label{def:truthfullness}
  $M$ is \textbf{truthful} if for all $R$, all $p_i$, all $r_i \neq p_i$,
    \[ M(p_i, R_{-i}; p_i)_i \geq M(r_i, R_{-i}; p_i)_i .\]
  $M$ is \textbf{strictly truthful} if the inequality is always strict.
\end{definition}
Verbally, the mechanism is (strictly) truthful if the probability of selecting $i$ is (uniquely) maximized when $i$ reports $r_i = p_i$, fixing all others' reports.
Note the probability here is taken both over the randomness of the mechanism and the randomness of the events, according to $i$'s beliefs.

Meanwhile, we say a mechanism is $\gamma$-approximately truthful if it is a strictly dominated strategy to make any report $r_{it}$ with $|r_{it} - p_{it}| > \gamma$.
Recall that, for a fixed $p_i$, the report $\hat{r}_i$ \emph{strictly dominates} $r_i$ if for all $R_{-i}$, we have $M(\hat{r}_i,R_{-i};p_i) > M(r_i, R_{-i};p_i)$.
We say that $r_i$ is \emph{strictly dominated} if such an $\hat{r}_i$ exists, and $r_i$ is \emph{undominated} otherwise.
Observe that any mechanism is vacuously $\gamma$-approximately truthful if $\gamma \geq 1$.
\begin{definition}
  \label{def:approx-truthful}
  A mechanism $M$ is \textbf{$\gamma$-approximately truthful} if for all $p_i$, (i) there exists an undominated report, and (ii) for all undominated reports $r_i$, we have $\|r_i - p_i\|_{\infty} \leq \gamma$.
\end{definition}

Our notion of approximate truthfulness is stronger than the typical one which only requires the utility to be approximately optimized by truthful reporting (e.g. Dwork and Roth \citet[Definition 10.2]{dwork2014algorithmic}).
This type of approximate truthfulness is relatively easy to achieve in our context, as one could simply mix $M$ with the uniform distribution to dampen the incentives.
By contrast, our definition requires the approximation to be in the report space itself, rather than the utility.
This stronger condition is crucial to our results, as our accuracy and no-regret guarantees rely heavily on any undominated report being close to truthful.
When $M$ is continuous in the reports, our definition implies the weaker version.

\subsection{Accuracy}
We use the (unknown) ground truth probabilities $\vec{\theta}$ to define the accuracy of the forecasters.
The goal of the mechanism is to pick a forecaster with approximately optimal accuracy.
Following \citet{witkowski2101incentive}, we use the following measure of accuracy.
Note that a forecaster's accuracy is not dependent on the outcomes of the events.
\begin{definition}
  \label{def:accuracy}
  Each forecaster $i$'s \textbf{accuracy} is $a_i = 1 - \frac{1}{m}\sum_{t=1}^m (p_{it} - \theta_t)^2$.
  We say a forecaster $i$ is \textbf{$\epsilon$-optimal} if $a_i \geq \max_j a_j - \epsilon$.
\end{definition}

We call a mechanism $(\epsilon,\delta)$-accurate if it selects an $\epsilon$-optimal forecaster except with probability $\delta$.
For a non-truthful mechanism, this definition is subtle for two reasons.
First, we would like to select a forecaster whose true beliefs $p_i$ are accurate, regardless of their strategic reports $r_i$.
Second, the mechanism's accuracy guarantee presumably assumes something about what forecasters are reporting, even if not truthful.
It depends on the solution concept of the game, e.g. ``the mechanism is accurate in equilibrium''.
Here, we will only assume undominated strategies as a solution concept.

\begin{definition}
  \label{def:mech-accurate}
   A mechanism $M$ is \textbf{$(\epsilon,\delta)$-accurate} in the setting defined by $(n,m,P,\vec{\theta})$ if for all $R$ consisting of undominated strategies, with probability at least $1-\delta$ over event outcomes $\vec y \sim \vec{\theta}$ and $i \sim M(R,\vec{y})$, the winner $i$ is $\epsilon$-optimal.
\end{definition}
   
The key question of this paper is: given $n$ forecasters and an accuracy goal of $(\epsilon,\delta)$, how many events $m$ are needed?
\begin{definition}
  \label{def:mdeltaeps}
   The \textbf{event complexity} of a mechanism $M$ is the function $m^*: \mathbb{N} \times [0, 1] \times [0, 1] \rightarrow \mathbb{N}$ such that, for all $n,\epsilon,\delta$,
   the output $m = m^*(n,\epsilon,\delta)$ is the smallest integer such that, for all $(P,\vec{\theta})$, the mechanism $M$ is $(\epsilon,\delta)$-accurate in the setting $(n,m,P,\vec{\theta})$.
\end{definition}

The \emph{nonstrategic event complexity} of a mechanism is its event complexity assuming access to the true beliefs $P$.
In other words, this is the ``first-best'' that could be achieved if there were no strategic considerations.
A central question for forecaster selection is the cost of informational asymmetry, i.e. the event complexity gap between the strategic and nonstrategic settings.

The task in this paper of selecting an $\epsilon$-optimal forecaster is slightly different from the task in \citet{witkowski2101incentive}.
There, it is assumed that there exists an ``$\epsilon$-dominant'' forecaster $i$ with $a_i - \epsilon \geq \max_{j \neq i} a_j$.
The task here is weakly more difficult: if we have an $(\epsilon,\delta)$-accurate mechanism, it will necessarily select an $\epsilon$-dominant forecaster with probability $1-\delta$, satisfying the goal in that paper.
It turns out that their mechanism, ELF, solves the harder problem of selecting an $\epsilon$-optimal forecaster (\S~\ref{sec:elf}).
The biggest impact of this change is that our lower bound of $m^* = \Omega\left({\log(n)}/{\epsilon^2}\right)$, Theorem \ref{thm:main-lower-bound}, will apply to selecting $\epsilon$-optimal forecasters.
We do not have a lower bound for the \citet{witkowski2101incentive} $\epsilon$-dominant variant of the problem.

\subsection{The quadratic scoring rule}
As in \citet{witkowski2101incentive}, we will focus on the quadratic scoring rule for assessing forecasts.
In principle our results could be extended to other scoring rules, a question we leave for future work.%
\begin{definition}
  \label{quadratic_scoring_rule}
  The \textbf{quadratic scoring rule} is the function $S: [0,1] \times \{0,1\} \to [0,1]$ defined by $S(q,y) = 1 - (y - q)^2$.
\end{definition}
The quadratic score is an example of a \emph{strictly proper scoring rule} $S': [0,1] \times \{0,1\} \to \reals$.
Such rules guarantee that one maximizes expected score, according to a belief $p_{it}$, by reporting $p_{it}$.
The quadratic or ``Brier'' score was introduced by \citet{brier1950verification}, and more on proper scoring rules can be found in \citet{gneiting2007strictly}.
As in \citet{witkowski2101incentive}, the quadratic score is closely related to our definition of accuracy and to our mechanisms, which reward forecasters based on their quadratic scores.
It is also a simple transformation of the squared loss, one of the most common losses in machine learning.
In particular, a forecaster's expected average quadratic score is equal to their accuracy $a_i$ (Definition \ref{def:mech-accurate}), up to a constant depending on the event variances.
\begin{lemma}[\citet{witkowski2101incentive}] \label{lemma:quad-accuracy}
  From a ground truth perspective, forecaster $i$'s expected average quadratic score when truthful is
    \[ \E_{\vec y\sim \vec\theta}\left[ \frac{1}{m} \sum_{t=1}^m S(p_{it}, y_t) \right] = a_i - C_{\vec{\theta}} , \]
  where $C_{\vec{\theta}} = \frac{1}{m} \sum_{t=1}^m \theta_t (1 - \theta_t)$.
\end{lemma}
\begin{proof}
By definition, the expected score of $i$ on event $t$ is
\begin{align*}
  \E [S(p_{it},y_t)]
  &= \theta_t \left[1 - (1 - p_{it})^2\right] + (1 - \theta_t) \left[1 - p_{it}^2\right]  \\
  &= \theta_t \left[2p_{it} - p_{it}^2\right] + (1-\theta_t)\left[1 - p_{it}^2\right]  \\
  &= 1 - \theta_t + 2 \theta_t p_{it} - p_{it}^2  \\
  &= 1 - \theta_t + \theta_t^2 - (p_{it} - \theta_t)^2 ~.
\end{align*}
Averaging over the $m$ events gives the result.
\end{proof}

This result suggests that the accuracy goal is perfectly aligned with selecting a forecaster based on total quadratic score.
We next discuss the baseline of directly using this total score.

\subsection{The Simple Max baseline}
\label{sec:simple-max}
A straightforward selection mechanism often used in practice is the \emph{Simple Max mechanism}. For this mechanism, we assign each forecaster a score $f_{it} = S(r_{it}, y_t)$ for each event. Then, we assign their final score by summing these over all events, $F_i = \sum_{t=1}^m f_{it}$. Finally, we choose the forecaster with the highest cumulative score as the overall winner.

\paragraph{Truthfulness}
As observed by \citet{witkowski2101incentive} and \citet{aldous2019prediction} and discussed in depth by \citet{lichtendahl2007probability}, this mechanism is generally not truthful.
To illustrate, consider three forecasters and one event whose true probability is $0.5$.
Alice predicts $0.5$, but Bob predicts $0.9$ and Charlie predicts $0.1$.
Observe that Alice \emph{cannot win}, despite being the best forecaster by far: either the event occurs (Bob wins) or it doesn't (Charlie wins).
Alice can only win by predicting either $0$ or $1$, raising her probability of winning from $0$ to $0.5$.

\paragraph{Event complexity}
Although the simple max mechanism is not truthful, it is worth studying its nonstrategic event complexity as a baseline.
The nonstrategic event complexity $m^*$, by analogy to Definition~\ref{def:mdeltaeps}, denotes the minimum number of events required to select an $\epsilon$-optimal forecaster with probability $1-\delta$, but now assuming access to the true beliefs $P$.

\begin{prop} \label{prop:simple-max}
  The Simple Max mechanism has a nonstrategic event complexity of 
  \[ m^* \leq \frac{2\log\left(\frac{n}{\delta}\right)}{\epsilon^2} ~.\]
\end{prop}
\begin{proof}
We have $F_i = \sum_t S(p_{it}, y_t)$. By Lemma \ref{lemma:quad-accuracy}, $\E[F_i] = m a_i - c$ for some constant $c = m C_{\vec \theta}$.
Let $i^*$ be the highest accuracy forecaster.
For any non-$\epsilon$-optimal forecaster $j$ we have $a_i > a_j + \epsilon$, so $\E[F_i] - \E[F_j] > m \epsilon$.
Therefore, a non-$\epsilon$-optimal forecaster cannot win if $F_{i^*} \geq \E[F_{i^*}] - \tfrac{m\epsilon}{2}$ and $(\forall j \neq i^*)$ $F_j \leq \E[F_j] + \tfrac{m\epsilon}{2}$.
We therefore show this fails to happen with probability at most $\delta$.

Because $F_i$ is the sum of independent variables bounded in $[0,1]$, by Hoeffding's inequality, we have for all $i$
\begin{align*}
  \Pr\left[F_i - \E[F_i] > \frac{m\epsilon}{2}\right] < e^\frac{-m \epsilon^2}{2} ~, \\
  \Pr\left[\E[F_i] - F_i > \frac{m\epsilon}{2}\right] < e^\frac{-m \epsilon^2}{2} ~.
\end{align*}
Setting $m \geq \frac{2\log(n/\delta)}{\epsilon^2}$, we have $\Pr[F_{i^*} < \E[F_{i^*}] - \tfrac{m\epsilon}{2}] \leq \tfrac{\delta}{n}$; and for each $j \neq i^*$, $\Pr[F_j > \E[F_j] + \tfrac{m\epsilon}{2}] \leq \tfrac{\delta}{n}$.
By the union bound, a non-$\epsilon$-optimal forecaster is able to win with probability at most $\delta$.
\end{proof}
Meanwhile, in Theorem \ref{thm:main-lower-bound}, we will use a reduction from agnostic PAC learning to prove that all mechanisms have nonstrategic event complexity $m^* = \Omega\left(\frac{\log(n)}{\epsilon^2}\right)$.
Therefore, Simple Max is essentially the best possible.
In particular, we can select the best forecaster using a number of events only logarithmic in $n$, the number of forecasters---if truthfulness is not required.
However, the state of the art for truthful mechanisms is significantly worse: \citet{witkowski2101incentive} gives the only bound to our knowledge, showing that their truthful ELF mechanism achieves $m^* = O\left(\frac{n^2 \log(n)}{\epsilon^2}\right)$ events.
We next attempt to close the gap.

\section{A Tight Analysis of ELF} \label{sec:elf}
The ELF mechanism $\elf$ is a truthful forecaster selection mechanism introduced by \citet{witkowski2101incentive}. 
For each event $t=1,\dots,m$, we use a lottery to award a point to a single forecaster.
Forecaster $i$ is chosen with probability
\begin{equation}
  f_{it} = \frac{1}{n} + \frac{1}{n} \left(S(r_{it}, y_t) - \frac{\sum_{j\neq i} S(r_{jt}, y_t)}{n-1} \right) ~.
  \label{eqn:wagering}
\end{equation}
Let $F_{it}$ be the indicator function for forecaster $i$ getting the point for event $t$, and $F_i = \sum_t F_{it}$ be the random variable equal to the number of points forecaster $i$ obtains. Then, ELF chooses $\argmax_i F_i$ as the winner, breaking ties uniformly. 

Equation \ref{eqn:wagering} is adapted from single-round wagering mechanisms \cite{lambert2008self}, with the idea that each forecaster wagers $\tfrac{1}{n}$ units on her report, with the chance to win back between $0$ and $\tfrac{2}{n}$ units (using that scores are in $[0,1]$).
The units are then converted into a probability of winning the point.
By increasing her own quadratic score a forecaster increases the chances she wins the point for round $t$, while uniformly decreasing the chances any other forecaster wins that point.

Theorem 6 of \citet{witkowski2101incentive} shows $\elf$ to be strictly truthful.
Theorem 8 of the same paper also shows\footnote{More precisely, their result is slightly weaker as stated: if there exists an $\epsilon$-dominant forecaster $i$, i.e. one with $a_i > \max_{j\neq i} a_j + \epsilon$, then it is selected with probability $1-\delta$. But essentially the same argument shows that $\elf$ unconditionally guarantees to select an $\epsilon$-optimal forecaster, with the same number of events.} that $\elf$ chooses an $\epsilon$-optimal forecaster with probability $1 - \delta$ for all 
\[ m \geq \frac{2(n-1)^2}{\epsilon^2} \log \left(\frac{4(n-1)}{\delta}\right) ~.\]
We begin by showing that $\elf$'s event complexity can be lowered from a quadratic dependence on $n$ to a linear one.

\subsection{Upper bound}
Our proof of the upper bounds follows the same outline as that of \citet{witkowski2101incentive}, but uses a tighter concentration bound at a key moment. 
Let $i$ be the best forecaster and $j$ be any non-$\epsilon$-optimal forecaster.
\citet{witkowski2101incentive} use Hoeffding's inequality to bound the probability that $i$'s total score is much below its expectation, or $j$'s is much above.
Hoeffding's is tight when the variance of each independent variable in a sum is $\Omega(1)$.
However for ELF, the probability of $i$ winning a point on round $t$ is bounded in $[0,\tfrac{2}{n}]$.
In such cases, Bernstein's inequality gives a tighter bound than Hoeffding's, because it takes into account the variance of the sum as well as the bound on the individual variables.
In particular, while a general sum of $m$ Bernoullis could have worst-case variance $\Omega(m)$, $i$'s total score $F_i$ has a variance bounded by $\tfrac{2m}{n}$, a factor of $n$ smaller.
This translates to a factor-$n$ improvement in the bound, which we prove in Appendix \ref{sec:appendix-elf-upper}.

\begin{theorem}
  For $n \geq 3$, $\elf$ has an event complexity given by
  \label{thm:elf-upper-bound}
    $m^*(n, \epsilon, \delta) \leq \frac{5(n-1)}{\epsilon^2}\log\left(\frac{4(n-1)}{\delta}\right)$.
\end{theorem}

\subsection{Lower bound}
To prove a lower bound for ELF, it suffices to consider a case with a single perfect forecaster with $p_1 = \vec{\theta} = (1,\dots,1)$, and $n-1$ terrible forecasters with $p_j = (0,\dots,0)$.
Unfortunately, despite forecaster $1$'s clear advantage, she is only chosen to gain a point with probability $\frac{2}{n}$ per round, while all other forecasters have a chance of slightly under $\frac{1}{n}$.
Using a balls-in-bins result implies that, for $m < O(n \log n)$, some terrible (and lucky) forecaster is likely to have more points than forecaster $1$.
This scenario yields the following bound.

\begin{theorem}
\label{thm:elf-lower-bound}
 For any $\delta < \frac{1}{2}$, $\epsilon < 1$, and all sufficiently large $n$, $\elf$ has event complexity $m^*(n, \epsilon, \delta) > \frac{n}{4} \log n$.
\end{theorem}

We present the proof in Appendix \ref{sec:appendix-elf-lower}.
Additionally, in Appendix \ref{sec:appendix-general-elf}, we study a broader class of truthful and symmetric "ELF-like" mechanisms that independently award a point to a forecaster for each event, and then choose the one with the highest score.
Leveraging connections to wagering mechanism, we can extend the lower bound of Theorem \ref{thm:elf-lower-bound} to this entire class of mechanisms.

The best-known exactly-truthful mechanisms therefore achieve an event complexity of $\Theta(n \log n)$.
To improve on this bound, we first introduce a relaxation of the truthfulness requirement.

\section{An Approximately Truthful Mechanism: FTRL}
\label{sec:approx-truthful}

In this section, we show how to achieve approximate truthfulness in the strong sense of Definition~\ref{def:approx-truthful}: when $R$ consists of undominated reports, $|r_{it} - p_{it}| \leq \gamma$ for all $i,t$.
To do so, we turn to machine learning algorithms, a natural choice given that our definition of accuracy (Definition~\ref{def:mech-accurate}) strongly resembles PAC (probably approximately correct) learning guarantees.

The simplest learning algorithm would be Simple Max, which picks the forecaster with the best total quadratic score.
A key incentive problem with Simple Max, exhibited in \S~\ref{sec:simple-max}, is its sensitivity to the input: a small change in a report $r_{it}$ can completely change a forecaster's probability of winning.
In machine learning, \emph{regularization} is often used to make an algorithm's decisions less sensitive while still retaining accuracy.
Combining a regularizer with Simple Max yields a Follow the Regularized Leader (FTRL) algorithm (e.g., \cite{hazan2019introduction,shalev2011online}), the class we consider.

A canonical example of an FTRL algorithm is Multiplicative Weights.\footnote{The role of regularization in Multiplicative Weights is not clear from the definition, but will be discussed in the next section.}
Given a tunable parameter $\eta>0$, Multiplicative Weights $M^*_\eta$ selects forecaster $i$ with probability
\begin{equation}
  \label{eq:mult-weights}
  M^*_\eta(R,\vec y)_i = \frac{\exp\left(\eta \sum_{t=1}^m S(r_{it},y_t)\right)}{\sum_{j=1}^n \exp\left(\eta \sum_{t=1}^m S(r_{jt},y_t)\right)} ~.
\end{equation}
For intuition on its approximate truthfulness, consider what happens when we fix everything but $r_{it}$ and $y_t$, and take the expected value with respect to $y_t\sim p_{it}$.
For small enough $\eta$, the denominator of (\ref{eq:mult-weights}) barely changes, and the numerator is nearly linear in $S(r_{it},y_t)$.
So forecaster $i$'s utility function will behave similarly to $\E_{y_t\sim p_{it}} S(r_{it},y_t)$, which is maximized by truthful reporting by properness of the quadratic scoring rule.
Furthermore, if $S$ is concave enough (and it is), then the utility-maximizing report $r^*_{it}$ on round $t$ will satisfy $|r^*_{it}-p_{it}|\leq O(\eta)$.
Further characterizing the optimal \emph{ex ante} report (in expectation over all outcomes) is nontrivial, but a small extension suffices to to show that any undominated report $r_i$ satisfies $\|r_i - p_i\|_{\infty} \leq \gamma$.

We next formalize and generalize this analysis approach.
With a curvature assumption on the regularizer, Condition \ref{cond:regularizer}, we show that all FTRL algorithms yield $O(\eta)$-approximate truthfulness up to constants depending on the regularizer.
See \S~\ref{sec:discussion} for a discussion of the related algorithm, Follow the Perturbed Leader.

\subsection{FTRL}
Follow the Regularized Leader (FTRL) is a common class of learning algorithms for prediction with expert advice~\cite{shalev2011online}.
Although these algorithms are designed to select a sequence of experts (forecasters) over a series of rounds, we will also be able to use them as a static selection mechanism by simply applying them to the entire batch of $m$ events.

A \emph{regularizer} is a strictly convex, differentiable%
\footnote{Following e.g.~\citet{mhammedi2018constant}, we say a regularizer $\R$ is differentiable on $\Delta_n$ if its directional derivative $\R'(\pi;x)$ is linear in $x$ on the subspace $\{x\in\reals^n \mid \sum_i x_i = 0\}$.}
function $\R:\Delta_n\to\reals$.
For $\eta>0$, the FTRL mechanism $M_{\R,\eta}$ chooses the forecaster distribution according to
\begin{equation}
  \label{eq:1}
  M_{\R,\eta}(\R,\vec y) \in \argmax_{\pi\in\Delta_n}\left\{ \eta \sum_{i=1}^n \pi_i \sum_{t=1}^m S(r_{it},y_i) - \R(\pi) \right\}~.
\end{equation}
The conditions on $\R$ above imply that the choice of $M_{\R,\eta}(R,\vec y)$ is unique and can be written
\begin{equation}
  \label{eq:2}
  M_{\R,\eta}(R,\vec y) = \nabla C(\eta q)~,
\end{equation}
where the convex, differentiable function $C = \R^*$ is the convex conjugate of $\R$ and $q = q(R,\vec y) \in \reals^n$ with $q_i = \sum_{t=1}^m S(r_{it},y_t)$~\cite[consequence of Theorems 26.3 and 26.1]{rockafeller1997convex}.
Verbally, the mechanism considers the vector $q$ of total quadratic scores, scales it by $\eta$, and takes the gradient of $C$ at this point, yielding a distribution on forecasters.

An important example is Multiplicative Weights, given by $M^*_\eta := M_{\R,\eta}$ where the regularizer is negative entropy, $\R(\pi) = \sum_{i=1}^n \pi_i \log \pi_i$.
One can verify that here $C(x) = \R^*(x) = \log \sum_{i=1}^n \exp(x_i)$.
Furthermore, taking $\nabla C(\eta q)$ gives back exactly eq.~\eqref{eq:mult-weights}.

\subsection{Approximate Truthfulness} \label{subsec:ftrl-approx-truth}

We now show that FTRL with certain regularizers satisfies the strong notion of approximate truthfulness in Definition~\ref{def:approx-truthful}.
As we will see, while weaker than truthfulness, this guarantee is strong enough to ensure that the mechanism is accurate, and even no-regret in an online setting.

Fixing others' reports $R_{-i}$ and the realized outcomes $\vec{y}$, define $U_i(r_i) = M_{\R,\eta}(r_i, R_{-i},\vec y)_i$ for the probability forecaster $i$ wins the competition as a function of their report vector.
Then, we write $\partial_i C(\cdot)$ to refer to the partial derivative of $C$ with respect to its $i$th argument; $\partial_i^2 C(\cdot)$ for its second partial derivative with respect to the $i$th argument, and so on.

We therefore have $U_i(r_i) = \partial_i C(\eta \cdot q(r_i, R_{-i},\vec{y}))$, where $q$ is defined above.
Our proof relies on carefully controlling the curvature of $U_i$ as a function of each individual report $r_{it}$.
Consider the first two derivatives of $U_i$:
\begin{align}
  \nabla U_i &= \eta \cdot \partial^2_i C(\eta q) \cdot \nabla q~, \label{eq:grad-U}
  \\
  \nabla^2 U_i &= \eta^2 \cdot \partial^3_i C(\eta q) \cdot \nabla q \nabla q^\top + \eta \cdot \partial^2_i C(\eta q) \cdot \nabla^2 q ~. \label{eq:hessian-U}
\end{align}
For the quadratic score, we have $(\nabla q)_t = 2(y_t - r_{it})$ and $\nabla^2 q = -2I$, where $I$ is the $m\times m$ identity matrix.

To control the curvature of $U_i$, therefore, we must control the curvature of $C$, which leads to the following condition:

\begin{condition}\label{cond:regularizer}
  Given regularizer $\R$, let $C=\R^*$.
  Then $C$ is thrice differentiable, and:
  \begin{itemize}
      \item[(i)] There exists $\alpha>0$ such that $\partial^2_i C(x) \geq \alpha \, |\partial^3_i C(x)|$ for all $x\in\reals^n$ and $i\in \{1,\dots,m\}$.
      \item[(ii)] There exists $\beta>0$ such that $\log \left(\partial^2_i C(x)\right)$ is $\beta$-Lipschitz in $\|\cdot\|_{\infty}$ as a function of $x$, i.e. $\left|\log\frac{\partial^2_i C(x)}{\partial^2_i C(x')}\right| \leq \beta \|x - x'\|_{\infty}$.
  \end{itemize}
\end{condition}

Before continuing, let us verify that Multiplicative Weights $M^*_\eta$, which is the FTRL algorithm $M_{\R,\eta}$ with $\R(\pi) = \sum_{i=1}^n \pi_i \log \pi_i$, satisfies the condition.
\begin{lemma}\label{lem:mult-weights}
  $M^*_\eta$ satisfies Condition~\ref{cond:regularizer} with $\alpha = 2$ and $\beta = 3$.
\end{lemma}
\begin{proof}
  We have $C(x) = \R^*(x) = \log \sum_{i=1}^n \exp(x_i)$.
  Computing,
  \begin{align*}
    \partial_i C(x) &= \exp(x_i) \left(\textstyle\sum_{j=1}^n \exp(x_j)\right)^{-1},
    \\
    \partial^2_i C(x) &= \exp(x_i) \left(\textstyle\sum_{j\neq i}^n \exp(x_j)\right) \left(\textstyle\sum_{j=1}^n \exp(x_j)\right)^{-2},
    \\
    \partial^3_i C(x) &= \exp(x_i) \left(\textstyle\sum_{j\neq i}^n \exp(x_j)\right) \left(\textstyle\sum_{j\neq i}^n \exp(x_j) - \exp(x_i)\right) \left(\textstyle\sum_{j=1}^n \exp(q_j)\right)^{-3}.
  \end{align*}

  To check the conditions,
  \begin{align*}
    \frac {\partial^2_i C(x)} {|\partial^3_i C(x)|} &= \left(\textstyle\sum_{j=1}^n \exp(x_j)\right) \left|\textstyle\sum_{j\neq i}^n \exp(x_j) - \exp(x_i)\right|^{-1} \geq  2~,
    \\
    \frac {\partial^2_i C(x)} {\partial^2_i C(x')} &= \exp(x_i-x'_i) \left(\exp(x'_i) + \textstyle\sum_{j\neq i}^n \exp(x_j)\right)^2 \left(\exp(x_i) + \textstyle\sum_{j\neq i}^n \exp(x_j)\right)^{-2}~.
  \end{align*}
  Thus we may take $\alpha = 2$.
  For $\beta$, observe that the dual $C$ of any regularizer is 1-Lipschitz: $\|\nabla C\|_1 = 1$ as $\mathrm{dom}\,\R = \Delta_n$, and since $\|\cdot\|_1$ and $\|\cdot\|_\infty$ are dual norms, e.g.\ \citet[Lemma 2.6]{shalev2011online} gives $|C(x)-C(x')| \leq 1 \|x-x'\|_\infty$.
  Thus, we have $\left|\log \frac{\partial^2_i C(x)}{ \partial^2_i C(x')} \right| = \left|x_i-x_i' + 2C(x)-2C(x')\right| \leq |x_i - x_i'| + 2|C(x) - C(x')| \leq 3$ for $\|x-x'\|_{\infty} \leq 1$.
\end{proof}

\newcommand{\Ubari}[2][]{\overline U_i#1(#2)}  %
\newcommand{\Ubarit}[2][]{\overline U_{it}#1(#2)} %

To reason about the incentives of forecaster $i$, it will be convenient to write $\Ubari{r_i} := \E_{\vec{y} \sim p_i} U_i(r_i)$, that is, $i$'s expected utility over her beliefs on $\vec{y}$ and the mechanism's randomness.
We will also use versions of $U_i$ and $\overline U_i$ when we restrict attention to only round $t$.
Specifically, let $U_{it}(r_{it}):=U_i(r_i)$ just as function of $r_{it}$, for fixed values of $r_{it'}$, $t'\neq t$, and define $\Ubarit{r_{it}} := \E_{y_t\sim p_{it}} U_{it}(r_{it})$.
We suppress the dependence on $p_i$ as beliefs will be fixed and arbitrary throughout this section.
\begin{lemma}\label{lem:FTRL-concave}
  Let $\R$ satisfy Condition~\ref{cond:regularizer}(i) for $\alpha$, and suppose $C$ is strictly convex.
  For $\eta < \tfrac{\alpha}{2}$, for all $i\in[n],t\in[m]$, and all $R_{-i}$, the functions $\Ubarit{r_{it}}$ and $\Ubari{r_i}$ are strictly concave in $r_{it}$.
\end{lemma}
\begin{proof}
  Let $f(r_{it}) = U_i(r_{it},r_{i,-t})$, i.e., $U_i$ as a function of $r_{it}$.
  We have
  \begin{align*}
    f''(r_{it}) = \tfrac {d^2} {dr_{it}^2} U_i
    &= \eta^2 \partial^3_i C(\eta q) 4(y_t - r_{it})^2 + \eta \partial^2_i C(\eta q) (-2)
    \\
    &\leq \eta^2 |\partial^3_i C(\eta q)| 4 - 2 \eta \partial^2_i C(\eta q)
    \\
    &\leq 4\eta^2 \tfrac{1}{\alpha} \partial^2_i C(\eta q) - 2 \eta \partial^2_i C(\eta q)
    \\
    &=  2\eta(2\tfrac \eta \alpha - 1) \partial^2_i C(\eta q)~.
  \end{align*}
  Letting $z = 2\eta(2\tfrac \eta \alpha - 1)$, and noting $z<0$ by assumption on $\eta$, we have $f''(r_{it}) = z \partial^2_i C(\eta q)$.

  Let $v = q - \eta S(r_{it},y_t) e_i$, where $e_i$ is the $i$th indicator vector.
  By definition of $q$, $v$ is a constant with respect to $r_{it}$.
  As $C$ is strictly convex, the function $g:[0,\eta]\to\reals, a\mapsto C(v + a e_i)$ is strictly convex.
  From \cite[Corollary 1.3.10]{niculescu2006convex}, we conclude $g''(a) \geq 0$ for all $a\in[0,\eta]$, and the set $\{a \in [0,\eta] : g''(a)=0 \}$ cannot contain any open intervals.
  By construction, $g''(a) = \partial^2_i C(v + a e_i)$.
  Letting $h:[0,1]\to[0,\eta], r\mapsto \eta S(r,y_t)$, observe that $f''(r) = z g''(h(r))$ for all $r\in[0,1]$.
  By definition of the quadratic score, the $h$-image of any open interval contains an open interval.
  Thus, were $\{r\in[0,1] : f''(r) = 0\}$ to contain any open intervals, so would $\{a \in [0,\eta] : g''(a)=0 \}$, a contradiction.
  From \cite[Corollary 1.3.10]{niculescu2006convex}, we conclude strict concavity of $f$.

  Applying the above to $U_i$, we have strict concavity of $U_i$ as a function of $r_{it}$.
  Taking an expected value over $y_t$ and over all outcomes, respectively, $\Ubarit{r_{it}}$ and $\Ubari{r_i}$  are strictly concave in $r_{it}$ as the convex combination of strictly concave functions.
\end{proof}

The next result shows ``leave-one-out'' approximate truthfulness, i.e., that if a forecaster knew and fixed in advance everything about rounds other than $t$, she would still report $r_{it}$ approximately truthfully.
\begin{lemma}\label{lem:FTRL-approx}
  Let $\R$ satisfy Condition~\ref{cond:regularizer} for $\alpha,\beta$.
  Fix all reports but $r_{it}$ and all outcomes but $y_t$.
  Then for $\eta < \min(\tfrac{\alpha}{2},\tfrac{1}{\beta})$, letting $r^*_{it} = \argmax_{r\in[0,1]} \Ubarit{r}$, we have
  $|r^*_{it} - p_{it}| \leq \beta\eta+(\beta\eta)^2 < (\beta+1)\eta$.

\end{lemma}
\begin{proof}
  Because $\R$ satisfies Condition~\ref{cond:regularizer}(ii), $\partial^2_i C(x) > 0$ for all $i$ and $x$, so $C$ is strictly convex.
  Then by Lemma~\ref{lem:FTRL-concave}, it suffices to find the zero of the first derivative of $\E_{p_{it}} U_i$.
  \begin{align*}
    \E_{y_t \sim p_{it}} \tfrac {d} {dr_{it}} U_i
    &= \E_{y_t \sim p_{it}} \eta \partial^2_i C(\eta q) 2(y_t - r_{it})
    \\
    &= 2\eta \left( (1-p_{it}) \partial^2_i C(\eta q^0) ( - r_{it}) + p_{it} \partial^2_i C(\eta q^1) (1 - r_{it})\right)~.
  \end{align*}
  where $q^0,q^1$ are the values of $q$ when $y_t=0$ and $y_t=1$ respectively. 
  Setting the derivative to zero, we have
  \begin{align*}
    r_{it} &= \frac
             { p_{it} \partial^2_i C(\eta q^1) }
             { (1-p_{it}) \partial^2_i C(\eta q^0) +  p_{it} \partial^2_i C(\eta q^1) }
             = p_{it} \left(
             (1-p_{it}) \frac {\partial^2_i C(\eta q^0)} {\partial^2_i C(\eta q^1)} +  p_{it} \right)^{-1}~.\numberthis\label{eq:rit-pit-opt}
  \end{align*}
  Let $a = \partial^2_i C(\eta q^0) / \partial^2_i C(\eta q^1)$.
  By definition of $q$, $\|q^0 - q^1\|_{\infty} \leq 1$, i.e. if $y_t$ changes from $0$ to $1$ or vice versa, each person's total quadratic score changes by at most $1$.
  Condition \ref{cond:regularizer}(ii) now gives $|\log a| \leq \beta\eta$.
  From eq.~\eqref{eq:rit-pit-opt} we have $|\log r^*_{it} - \log p_{it}| \leq \max_{p\in[0,1]} |\log (p+(1-p)a)| \leq |\log a| \leq \beta\eta$.
  Without loss of generality, suppose $r^*_{it} \geq p_{it}$.
  Then $|r^*_{it} - p_{it}| \leq p_{it}(r^*_{it}/p_{it}-1) \leq p_{it} (e^{\beta\eta}-1) \leq e^{\beta\eta}-1 \leq \beta\eta + (\beta\eta)^2$,
  where we use the inequality $e^x \leq 1 + x + x^2$ for $x\in[0,1]$.
\end{proof}

Using leave-one-out approximate truthfulness, we can extend to her \emph{ex ante} preferences, obtaining our main approximate truthfulness result.
\begin{theorem}\label{thm:approx-truthful}
  Let $\R$ be any regularizer satisfying Condition~\ref{cond:regularizer} with $\alpha,\beta>0$.
  Then $M_{\R,\eta}$ is $(\beta+1)\eta$-approximately truthful for any $\eta < \min(\tfrac{\alpha}{2},\tfrac{1}{\beta})$.
\end{theorem}
\begin{proof}
  We first prove Definition \ref{def:approx-truthful}(i), existence of an undominated report.
  We will show that for any $p_i$, a best response to any $R_{-i}$ always exists.
  In particular, any best response is an undominated strategy.
  The convex function $C$ is differentiable, hence continuously differentiable~\cite[Theorem 25.5]{rockafeller1997convex}.
  By Equation (\ref{eq:2}), $i$'s utility is the $i$th component of $\nabla C(\eta q)$.
  Meanwhile, $q$ is a continuous function of $i$'s strategy $r_i$ for any fixed $R_{-i}$ and $\vec{y}$.
  Then, $i$'s expected utility is the $p$-convex combination of her utility for each $\vec{y}$, so it is also continuous.
  A continuous function on the compact set $[0,1]^m$, which is $i$'s strategy space, attains its maximum on that set.
  So $i$ has a best response.
  
  Now we show Definition \ref{def:approx-truthful}(ii), that all undominated reports are within $\gamma$ of $p_i$.
  Let $\gamma = (\beta+1)\eta$.
  Let $\hat r_i \in [0,1]^m$ with $\|p_i - \hat r_i\|_\infty > \gamma$.
  By definition of $\|\cdot\|_\infty$ we must have some $1\leq t\leq m$ such that $|\hat r_{it} - p_{it}| > \gamma$.
  Without loss of generality, assume $\hat r_{it} > p_{it} + \gamma$; the other case follows symmetrically.
  We will show that $\hat r_i$ is strictly dominated by $r_i^*$ given by $r^*_{it} = p_{it} + \gamma$ and $r_{i,-t}^* = \hat r_{i,-t}$, where we use ``$-t$'' to denote all entries of a vector except $t$.

  Let us now set $r_{i,-t} = \hat r_{i,-t}$ and fix all reports of $R$ except for $r_{it}$.
  Recalling the definitions of $\overline U_{it}$ and $\overline U_i$ above, observe that we have
  \begin{equation}
    \label{eq:bar-U}
    \Ubari{r_{it}, r_{i,-t}} = \E_{\vec y_{-t} \sim p_{i,-t}} \Ubarit{r_{it}|\vec y_{-t}}~,
  \end{equation}
  where we emphasize the dependence on $\vec y_{-t}$.
  In other words, as a function of just $r_{it}$, forecaster $i$'s probability of winning is simply the expected value of their probability of winning on round $t$ once the outcomes of all other rounds are revealed.

  For each $\vec y_{-t} \in\{0,1\}^{m-1}$, Lemma~\ref{lem:FTRL-approx} states that $\Ubarit{\cdot\,|\vec y_{-t}}$ is maximized by some $r(\vec y_{-t})$ with $|r(\vec y_{-t}) - p_{it}| \leq \gamma$.
  In particular, as $\Ubarit{\cdot\,|\vec y_{-t}}$ is continuously differentiable and strictly concave for all $\vec y_{-t}$, we must have $\Ubarit[']{r|\vec y_{-t}} < 0$ for all $r > p_{it}+\gamma$.
  In particular,
  $\Ubarit[']{\hat r_{it}|\vec y_{-t}} < 0$ for all $\vec y_{-t}$, yielding $\Ubari[']{\hat r_{it},r_{i,-t}} < 0$.
  The same logic and continuity of the derivative shows
  $\Ubari[']{p_{it}+\gamma,r_{i,-t}} \leq 0$.
  By strict concavity of $\overline U_i$, we conclude $\Ubari{r^*_i} = \Ubari{p_{it}+\gamma,r_{i,-t}} > \Ubari{\hat r_{it},r_{i,-t}} = \Ubari{\hat r_i}$.
  As this assertion holds for all values of $R$, the report $\hat r_i$ is strictly dominated by $r^*_i$.
\end{proof}

In particular, via Lemma \ref{lem:mult-weights}, we obtain approximate truthfulness for Multiplicative Weights.
\begin{corollary} \label{cor:MW-approx-truth}
  $M^*_\eta$ is $4\eta$-approximately truthful for any $\eta < \frac{1}{4}$.
\end{corollary}

\section{A Forecasting Mechanism with Optimal Event Complexity} \label{sec:main-accuracy}

We now utilize our approximate truthulness results for FTRL to obtain an order-optimal event complexity result.
In particular, Multiplicative Weights, when used as a forecasting competition mechanism, selects an $\epsilon$-optimal forecaster with $m = O\left({\log(n/\delta)}/{\epsilon^2}\right)$ events.
Our proof heavily relies on the results from Section~\ref{sec:approx-truthful} showing FTRL to be approximately truthful.
Put together, we show that, as long as forecasters play undominated strategies, FTRL satisfies two properties:
\begin{enumerate}
    \item[1.] For all $i,t$, $|r_{it} - p_{it}| \leq O(\epsilon)$.
    \item[2.] With $m \geq m^*$ events, the winner's accuracy is within $\epsilon$ of the best, with probability $1-\delta$.
\end{enumerate}
We emphasize that accuracy is defined in terms of true beliefs, regardless of reports.

Up to constant factors, this event complexity bound matches the nonstrategic event complexity of the Simple Max mechanism.
It also matches (for fixed $\delta$) the nonstrategic event complexity lower bound of Theorem  \ref{thm:main-lower-bound}.
In other words, if the goal is to select the best forecaster, then up to constant factors there is no cost of strategic behavior in this setting.
One can view approximate truthfulness as a nice-to-have guarantee that ultimately serves the goal of the competition in selecting a winner.
Of course, one may additionally seek exact truthfulness, which we discuss further in \S~\ref{sec:discussion}.

\subsection{Event Complexity Upper Bound}
We are now ready to prove the main result, an order-optimal event complexity in the presence of strategic behavior.
For this section, define:
\begin{itemize}
  \item $\empiricalquadscore_i = \frac{1}{m}\sum_{t=1}^m S(r_{it},y_t)$, i.e. $i$'s average quadratic score.
  \item $S_i = \E_{\vec{y} \sim \vec{\theta}}[\empiricalquadscore_i]$, i.e. expected average quadratic score.
  \item $S_i^* = \E_{\vec{y} \sim \vec{\theta}} \frac{1}{m} \sum_{t=1}^m S(p_{it}, y_t)$, the expected average quadratic score of $i$'s beliefs. Recall that $S^*_i$ is forecaster $i$'s accuracy $a_i^*$ plus a constant dependent only on $\theta$.
\end{itemize}

To show the event complexity, we show that if a forecaster's expected score $S^*_i$ with respect to their beliefs is far from the most accurate forecaster's expected score, then with high probability their actual score will also be far from the best forecaster's.
In particular, for sufficiently large $m$, their actual scores $Q_i$ are close to their expected scores $S_i$, and that for approximately truthful reports, $S_i$ is also close to $S^*_i$.
We first bound the deviation of each forecaster's actual score from their expected score given their reports.
\begin{lemma} \label{lemma:hoeff-conc}
  For any $i$, with probability at least $1 - \frac{\delta}{2n}$, we have $\empiricalquadscore_i - S_i \leq \sqrt{\log\left(\frac{2 n}{\delta}\right)/2m}$.
   The same statement holds replacing $\empiricalquadscore_i - S_i$ with $S_i - \empiricalquadscore_i$.
\end{lemma}
\begin{proof}
  $\empiricalquadscore_i$ is an average of $m$ independent random scores in $[0,1]$, with $\E[\empiricalquadscore_i] = S_i$.
  The result follows immediately from a standard Hoeffding bound \cite{mohri2018foundations}.
\end{proof}

Let $i = \argmax_j Q_j$ be the forecaster with the highest expected quadratic score. We can also bound the probability that the winner's score is far from the expected winner's score.

\begin{lemma} \label{lemma:mw-acc}
  With probability at least $1 - \frac{\delta}{2}$, the winner $i^*$ selected by Multiplicative Weights satisfies
    \[ \empiricalquadscore_{i^*} \geq \empiricalquadscore_i - \frac{\log(2n/\delta)}{m \cdot \eta} . \]
\end{lemma}
\begin{proof}
  Fix the total quadratic scores $\empiricalquadscore_1,\dots,\empiricalquadscore_n$ and let $\pi$ be the distribution over winners.
  Recall that $\pi_i := M^*_\eta(R,\vec y) = \frac{\exp(\eta m \empiricalquadscore_i)}{\sum_j \exp(\eta m \empiricalquadscore_j)}$.
  Consider any $j$ with $\empiricalquadscore_j < \empiricalquadscore_{i} - \frac{\log(2n/\delta)}{m \cdot \eta}$:
  \[
    \pi_j = \pi_{i} \exp\left(\eta m (\empiricalquadscore_j - \empiricalquadscore_{i})\right)  \leq \exp\left(\eta m (\empiricalquadscore_j - \empiricalquadscore_{i})\right)  \leq \frac{\delta}{2n} ~.
  \]
  Therefore, the total probability of selecting any such $j$ is bounded by $\frac{\delta}{2}$.
\end{proof}

Finally, we show that $S_i$ is close to $S^*_i$ for approximately truthful experts.
\begin{lemma} \label{lem:appx-truth-acc}
  For any $\gamma > 0$, if $\|r_j - p_j\|_{\infty} \leq \gamma$, then $|S_j - S^*_j| \leq 2\gamma$.
\end{lemma}
\begin{proof}
  First observe that the quadratic score is 2-Lipschitz in the report, as for all $r\in[0,1],y\in\{0,1\}$ we have $|\frac{d}{dr} S(r,y)| = |2(y - r)| \leq 2$.
  The result then follows.
\end{proof}

Now, we combine these bounds to show that an $\epsilon$-suboptimal forecaster's score will be far from the optimal forecaster's score with high probability.
\begin{theorem}\label{thm:ftrl-accuracy}
  The Multiplicative Weights mechanism $M^*_\eta$ with $\eta \leq \frac{\epsilon}{40}$ is $4\eta$-approximately truthful and selects an $\epsilon$-optimal forecaster with probability at least $1-\delta$, provided $m \geq \frac{5 \log(2n/\delta)}{\eta \cdot \epsilon}$.
  
  In particular, by choosing $\eta = \frac{\epsilon}{40}$, we obtain an event complexity bound of
    \[ m^* \leq \frac{200 \log(2n/\delta)}{\epsilon^2} . \]
\end{theorem}
\begin{proof}
  Let $B = \{ j : a_j < a_{i^*} - \epsilon \}$ be the set of non-$\epsilon$-optimal forecasters.
  
  Suppose $\eta \leq \frac{\epsilon}{40}$.
  Because Multiplicative Weights is $4\eta$-approximately truthful (Corollary \ref{cor:MW-approx-truth}), Lemma \ref{lem:appx-truth-acc} implies, for all $j$,
  \begin{equation}
     |S_j - S^*_i| \leq 8\eta \leq \frac{\epsilon}{5} ~.
     \label{eqn:appx-truth-bound}
  \end{equation}
  
  If $m \geq \frac{25 \log(2n/\delta)}{2 \epsilon^2}$, then by Lemma \ref{lemma:hoeff-conc} and a union bound, we have, except with probability $\frac{\delta}{2}$, the following: $\empiricalquadscore_{i} > S_{i} - \frac{\epsilon}{5}$ and, for all $j \in B$, $\empiricalquadscore_j < S_j + \frac{\epsilon}{5}$.
  In this event, we have for all $j \in B$:
  \begin{align*}
  \empiricalquadscore_j &< S_j + \frac{\epsilon}{5}  & \text{Lemma \ref{lemma:hoeff-conc}} \\
    &\leq S_j^* + \frac{2\epsilon}{5}  & \text{by (\ref{eqn:appx-truth-bound})}  \\
    &<    S^*_{i} - \frac{3\epsilon}{5}  & \text{definition of $B$, Lemma \ref{lemma:quad-accuracy}}  \\
    &\leq S_{i} - \frac{2\epsilon}{5} & \text{by (\ref{eqn:appx-truth-bound})}  \\
    &<    \empiricalquadscore_{i} - \frac{\epsilon}{5}    & \text{Lemma \ref{lemma:hoeff-conc}} .
  \end{align*}
  By Lemma \ref{lemma:mw-acc}, for $m \geq \frac{5 \log(2n/\delta)}{\eta \cdot \epsilon}$, except with probability $\frac{\delta}{2}$, the winner $i^*$ satisfies $\empiricalquadscore_{i^*} \geq \empiricalquadscore_i - \frac{\epsilon}{5}$.
  So by a union bound, with probability $1-\delta$, no member of $B$ is selected.
  We find that we just require $m \geq \frac{5 \log(2n/\delta)}{\eta \cdot \epsilon} \geq \frac{200 \log(2n/\delta)}{\epsilon^2}$.
\end{proof}

\subsection{Event complexity lower bound} \label{subsec:nonstrat-lower}
We now provide a matching lower bound (up to dependence on $\delta$) for the number of events required to select an $\epsilon$-optimal forecaster.
This lower bound applies to the non-strategic setting, where only the ground truth $\vec \theta$ is unknown.
In other words, the bound applies to ``first-best mechanisms'' that know all of the forecasters' true beliefs without needing to ask and are not constrained by truthfulness.

\begin{theorem} \label{thm:main-lower-bound}
  There exists $C > 0$ such that, for any mechanism $M$ and for all small enough $\epsilon$, the nonstrategic event complexity satisfies
   \[ m^*(n, \epsilon, \tfrac{1}{8}) \geq \frac{C~\log(n)}{\epsilon^2} . \]
\end{theorem}
We next sketch the main idea of the proof.
The full proof appears in Appendix \ref{sec:appendix-lower-bound}.

\paragraph{Warmup: two forecasters.}
Suppose Alice predicts $1$ and Bob predicts $0$ on all rounds.
The ground truth distributions are either all $\frac{1}{2} + \epsilon$ or all $\frac{1}{2} - \epsilon$.
As is well known from bounds on determining the bias of a coin, $\Omega(\frac{1}{\epsilon^2})$ events are required to determine which ground truth is the case.
The accuracy gap is $\Omega(\epsilon)$, giving an order-optimal $\frac{1}{\epsilon^2}$ lower bound.
(One may be tempted to have Alice and Bob predict $\frac{1}{2} \pm \epsilon$, but in such examples, the accuracy gap is generally $O(\epsilon^2)$ instead of $\Omega(\epsilon)$, yielding a suboptimal bound.)

\paragraph{Extension to $n$ forecasters: overview and challenges.}
To extend the approach to $n$ forecasters, we will turn to agnostic PAC learning with $n$ hypotheses and accuracy $\epsilon$, for which there is a suggestive sample complexity bound of $\Omega\left(\frac{\log n}{\epsilon^2}\right)$.
We use a natural approach of reducing from PAC learning: a dataset is like a set of events, where the label in $\{0,1\}$ is like an event outcome; and a hypothesis is like a forecaster who assigns a prediction to each data point (event).
If we have a forecasting competition that is efficient in picking the best forecaster, we should be able to use it to pick the best hypothesis and PAC-learn from a small amount of data.

The main obstacle to carrying through this reduction is that in PAC learning, accuracy of a hypothesis class is defined \emph{a priori} on the distribution over features $x$ and labels $y$.
In forecasting, the accuracy depends on which events are being predicted.
This is analogous to realizing all of the $x$ values first, then redefining the accuracy levels of all the hypotheses.
A forecasting competition may be able to identify the best forecaster conditioned on the realized events quite easily, although the original PAC problem is hard.\footnote{Suppose events are of two types, $A$ and $B$. Alice is always perfectly correct on events of type $A$ and perfectly wrong on type $B$, and vice versa for Bob. One distribution on event types is $(0.5 + \epsilon, 0.5 - \epsilon)$, and the other is the reverse. Now PAC learning is as difficult as distinguishing the bias of a coin. But if we are only given one randomly-drawn event, then selecting the best forecaster is easy: the accuracy gap is $1$.}

Fortunately for us, the main agnostic PAC lower bound distribution is uniform on $\X$, the space of possible data points.
So our PAC learner draws $2m$ samples from this hard uniform-marginal distribution, then trims the empirical distribution down so that it is of size $m$ and is perfectly uniform.
This implies that the forecasters' \emph{ex interim} accuracy, i.e. after defining the $m$ events but before the labels/event outcomes are realized from their conditional distributions, is equal to the \emph{a priori}.
So selection of a good forecaster is equivalent to selection of a good hypothesis -- which requires $m \geq \Omega\left(\frac{\log(n)}{\epsilon^2}\right)$.

Unfortunately, there are not enough samples $m$ relative to $|\X|$ for the argument just described to actually work.
We adapt the argument, showing that one can throw away the members of $\X$ that do not receive enough samples, and create an empirical distribution of events that is uniform on the remainder, while keeping an $O(\epsilon)$ difference between forecaster accuracy and hypothesis accuracy.
This modification allows us to complete the reduction.

\section{No-Regret Learning from Strategic Forecasters}
\label{sec:no-regret}

A growing literature in machine learning seeks to design learning algorithms with good performance guarantees even when the data points are chosen strategically by other agents.
Here, we consider an online learning setting where, each round, strategic experts report forecasts and the mechanism selects one of them as its own prediction.
The experts wish to be selected as many times as possible and may strategically misreport.

\citet{freeman2020no-regret}, building on \citet{roughgarden2017online}, give no-regret algorithms for the case where forecasters are \emph{myopic}.
In each round, myopic forecasters make whatever report they believe maximizes their chance of being selected in the next round.
The authors propose a learning algorithm based on ELF, which is truthful for such agents, but it is not known if it achieves vanishing regret.

But in general, strategic forecasters may misreport on certain rounds in order to affect their chance of being selected much later.
So the question remains: does there exist an incentive-compatible learning algorithm for general, non-myopic strategic forecasters?
We show that, in fact, FTRL is already such an algorithm.
The result holds for quadratic-score incentives, although we believe it may be extended.
Our proof relies on approximate truthfulness, together with the standard no-regret guarantees of FTRL.
The guarantee obtained is that, as long as forecasters play undominated strategies in the induced extensive-form strategic game, the algorithm's forecasts are competitive with the most accurate \emph{beliefs} of any expert.
We formalize the setting next.

\subsection{Regret and incentives}

We would like to find an online learning algorithm which achieves low regret with respect to the true beliefs of experts.
This notion is captured as follows.
\begin{definition}
  Given a forecasting competition mechanism $M$ and reports $R$, let $\pi^t = M(R_{1..t-1},\vec y_{1..t-1}) \in \Delta_n$ be the probability distribution over forecasters output by the mechanism after the first $t$ rounds.
  The regret of $M$ with respect to beliefs $P \in [0,1]^{n\times T}$ is
  \begin{equation}
    \Reg(M) =
    \argmax_{i\in[n]} \sum_{t=1}^T S(p_{it},y_t) - \sum_{t=1}^T \E_{i\sim \pi^t} S(r_{it},y_t)~.
  \end{equation}
\end{definition}
\noindent
(For this section, we let $T=m$ be the number of rounds / events, to match typical notation in online learning.)

There are several possible incentive structures one could consider for the experts.
\citet{freeman2020no-regret} and \citet{roughgarden2017online} consider the case where experts wish to maximize the internal weight given to them by the algorithm.
\citeauthor{freeman2020no-regret} consider normalized weights, where this weight can be interpreted as the probability an expert is ``chosen'' on any given round.
Notably, they only give a truthful no-regret algorithm for the \emph{myopic} case, where on round $t$ experts only care about their weight on round $t+1$.
We consider a general form which can capture such non-rational (i.e. time-inconsistent) preferences as well as long-term and rationality-compatible incentives.

A forecaster $i$'s utility is specified, for each round $t$, by a set of nonnegative constants $\{\uco_{it}^s\}_{s=t+1}^{T+1}$, not all zero.\footnote{We assume that at each round the forecaster cares at least slightly about being selected at \emph{some} future point. For this to hold on round $T$, we suppose the mechanism also makes a selection at round $T+1$, although this does not impact the regret.}
Her utility function at round $t$, for fixed reports $R$, is then
\begin{equation}
  \label{eq:online-learning-incentives}
  u_{it}(\pi^{t+1},\ldots,\pi^T) = \sum_{s=t+1}^{T+1} \uco_{it}^s \pi^s_i
\end{equation}
where $\pi^t$, the mechanism's distribution over experts at time $t$, is a function of the reports $R_{1..t-1}$ and outcomes $\vec{y}_{1..t-1}$ on rounds $1,\dots,t-1$.
Her perceived expected utility $U_{it}$ at time $t$ is defined to be the expectation of $u_{it}$ over her internal beliefs $p_i$ given reports and outcomes on rounds $1,\dots,t-1$.
We note that beliefs $p_i$ are immutable, i.e. do not change over time or update based on others' beliefs and actions.

Let $(R_{-i})_{t..T}$ denote the fixed reports of all forecasters other than $i$ on rounds $t,\dots,T$ (we have suppressed dependence on previous reports and outcomes).
A forecaster $i$'s strategy at time $t$ is a plan of reports $\strat_{it} = (r_{it}^s)_{s=t}^{T}$, chosen with a goal of maximizing $U_{it}(\strat_{it}, (R_{-i})_{t..T})$.
We say $\strat_{it}$ is \emph{strictly dominated} if there exists $\strat_{it}'$ with $U_{it}(\strat_{it}', (R_{-i})_{t..T}) > U_{it}(\strat_{it}, (R_{-i})_{t..T})$ for all fixed $(R_{-i})_{t..T}$.
This is a natural extension of the definition of incentive compatibility for forward-looking experts by \citet{freeman2020no-regret}.
We let $\strat_i = (\strat_{it})_{t=1}^T$ be a strategy of $i$ for the entire game, inducing realized reports $r_{it} = r_{it}^t$ for all $t$.
We say $\strat_i$ is undominated if all of its components $\strat_{it}$ are.
A mechanism is $\gamma$-\emph{approximately truthful} if for all undominated $\strat_i$, the induced reports $|r_{it} - p_{it}| \leq \gamma$ for all $t$.

This model can include preferences that are inconsistent across time, including the myopic preferences studied by \citet{freeman2020no-regret} given by $\uco_{it}^s = 1$ if $s=t+1$ and $0$ otherwise.
One could also consider discounted rewards, taking $\uco_{it}^s = \beta^{s-t}$ for $\beta \in (0,1)$.
In these inconsistent cases, a forecaster's plan at time $t$ for the round $s>t$, $r_{it}^s$, may not match her decisions at time $s$, $r_{is}^s$.

To model consistent ``rational'' preferences, one would require $\uco_{it}^s = \uco_{i1}^s$ for all $t \leq s$.
In this case, we have a game on $n$ players where $i$'s utility function is $\sum_{s=1}^{T+1} \uco_{i1}^s \pi_i^s$.
For instance, if $\uco_{it}^s = 1$ for all $s,t$, then forecasters wish to maximize the expected number of rounds they are chosen.
In the consistent case, optimizing utility in round one is compatible with optimizing utility on each other round, so one can take simply take $i$'s strategy to be the set of reports $r_i = (r_{it})_{t=1}^T$, as her plan for each round can be assumed to be consistent.
In this case, undominated strategies correspond to the usual definition for a simultaneous-move game where players commit to the reports $R$.
This simultaneous-move formalization of strategies, where $R_{-i}$ is fixed and $i$ responds, is also the approach of~\citet{freeman2020no-regret} to non-myopic incentive compatibility in online learning (see their Definition D.1).
We discuss extensive-form solution concepts in \S~\ref{sec:discussion}.

\subsection{Achieving no-regret}
FTRL is well-known to achieve no-regret with respect to the reports of the experts.
We would instead like a guarantee with respect to the \emph{beliefs} of the experts.
We give such a guarantee now, which holds for any strategies of the experts that are not strictly dominated.
Just as with our event complexity bound for forecasting competitions, the proof combines the non-strategic no-regret guarantee of FTRL with the approximate truthfulness of the algorithm in each time step.
We present the proof of approximate truthfulness in Appendix \ref{sec:appendix-online-approx-truth}. We omit it here as it simply applies the approach developed in Section \ref{subsec:ftrl-approx-truth}.

\begin{lemma} \label{lemma:online-approx-truth}
  Let the regularizer $\R$ satisfy Condition~\ref{cond:regularizer} for $\alpha,\beta>0$.
  Then $M_{\R,\eta}$ is an $(\beta+1)\eta$-approximately truthful online learning algorithm for any $\eta < \min(\tfrac{\alpha}{2}, \tfrac{1}{\beta})$.
\end{lemma}

We can now prove our main result for online learning from strategic experts, namely that FTRL still achieves no-regret.
Define $D_\R = \max_{\pi,\pi'\in\Delta_n} (\R(\pi)-\R(\pi'))$.
 
\begin{theorem}\label{thm:no-regret}
  Let the regularizer $\R$ be strongly convex\footnote{A differentiable convex function $f$ is strongly convex in norm $\|\cdot\|$ if for all $x,y\in\mathrm{dom} f$ we have $f(y)-f(x) \geq \nabla f(x)\cdot(y-x) + \tfrac 1 2 \|x-y\|$.} in the $L_1$ norm.
  Let $\R$ satisfy Condition~\ref{cond:regularizer} for $\alpha,\beta>0$ and let $T\geq \max(\tfrac 1 {\alpha^2}, \tfrac \beta 2)D_\R$.
  With choice of $\eta = \sqrt{D_\R / 2(\beta+2)T}$, we have for any beliefs $P\in[0,1]^{n\times T}$ and utilities $\{c_{it}^s\}$, for all strategy profiles consisting of undominated strategies, $\Reg(M_{\R,\eta}) \leq 2\sqrt{2(\beta+2)D_\R T}$.
\end{theorem}
We note that the result extends to regularizers that are strongly convex in other norms, with the usual additional factors in the regret~\cite{shalev2011online}.
\begin{proof}
  Standard regret guarantees for FTRL, such as \citet[Theorem 5.2]{hazan2019introduction} or \citet[Theorem 2.11]{shalev2011online}, give the following regret guarantee where the benchmark is the reports (not the beliefs):
  \begin{equation}
    \argmax_{i\in[n]} \sum_{t=1}^T S(r_{it},y_t) - \sum_{t=1}^T \E_{i\sim \pi_t} S(r_{it},y_t)
    \leq
    2 \eta T + \frac 1 \eta D_\R~.
  \end{equation}
  By Lemma~\ref{lemma:online-approx-truth}, for $\eta < \min(\tfrac{\alpha}{2},\tfrac{1}{\beta})$ we have $|p_{it}-r_{it}| \leq (\beta+1)\eta$ for all $i,t$.
  As the quadratic score is $2$-Lipschitz (see Lemma \ref{lem:appx-truth-acc}), $|S(r_{it},y_t) - S(p_{it},y_t)| \leq 2(\beta+1)\eta$.
  Therefore,
  \begin{align*}
    &\argmax_{i\in[n]} \sum_{t=1}^T S(p_{it},y_t) - \sum_{t=1}^T \E_{i\sim \pi_t} S(r_{it},y_t)
    \\
    &\leq \argmax_{i\in[n]} \sum_{t=1}^T \left(S(r_{it},y_t) + 2(\beta+1)\eta\right) - \sum_{t=1}^T \E_{i\sim \pi_t} S(r_{it},y_t)
    \numberthis\label{eq:no-regret-approx}
    \\
    &= \argmax_{i\in[n]} \sum_{t=1}^T S(r_{it},y_t) - \sum_{t=1}^T \E_{i\sim \pi_t} S(r_{it},y_t) + 2(\beta+1)\eta T
    \\
    &\leq
    2(\beta + 1 + 1)\eta T + \frac 1 \eta D_\R~.
  \end{align*}
  Taking $\eta = \sqrt{D_\R / 2(\beta+2)T}$ gives the result, as long as we have $\eta < \min(\tfrac{\alpha}{2},\tfrac{1}{\beta})$, as ensured by our bound on $T$.
  (To check, $1/\eta^2 \geq 2(\beta+2)\max(\tfrac 1 {\alpha^2}, \tfrac \beta 2) \geq \max(\tfrac 4 {\alpha^2}, \beta^2) = 1/\min(\tfrac {\alpha}{2}, \tfrac {1}{\beta})^2$.)
\end{proof}

As a brief aside, if one wishes to bound an alternative notion of regret where the algorithm is judged by the beliefs of its chosen expert, rather than their reports, then one simply picks up another additive $2(\beta+1)\eta T$ in eq.~\eqref{eq:no-regret-approx}, less than doubling the final regret bound.

Turning finally to Multiplicative Weights, we have that $\R$ is $1$-Lipschitz in $L_1$ norm.
Meanwhile, $D_\R = \log n$, and $\alpha=1/2$ and $\beta=3$ from Lemma~\ref{lem:mult-weights}.
From Theorem~\ref{thm:no-regret}, setting $\eta = \sqrt{\log n/10 T}$ gives the following.
\begin{corollary}
  For $T \geq 8$, for an appropriate choice of $\eta$, we have $\Reg(M^*_\eta) \leq 2\sqrt{10 T \log n}$.
\end{corollary}

\section{Discussion}
\label{sec:discussion}

We conclude with a few observations and open problems.

\paragraph{Follow the Perturbed Leader}
A natural alternative to an explicit regularizer in FTRL is to instead add noise to the total scores and then choose the maximum, an approach called Follow the Perturbed Leader (FTPL).
When the noise follows the Laplace distribution, this approach corresponds to the Report Noisy Max mechanism from differential privacy~\citep{dwork2014algorithmic}, which is well-known to provide approximate truthfulness in the weaker sense that forecasters will not gain much by deviating.
In Appendix \ref{sec:appendix-lower-bound}, we show that Report Noisy Max also satisfies our stronger notion of approximate truthfulness in undominated strategies.
The result is another mechanism that, like Multiplicative Weights, achieves optimal event complexity.
On the one hand, this result is unsurprising given the known equivalence of FTPL to FTRL for some choice of regularizer~\cite{kalai2005efficient, abernethy2017online}.
On the other, it suggests the robustness of our findings, and provides some intuition for why approximate truthfulness holds.

\paragraph{Other regularizers}
While we carefully study negative entropy as the regularizer, another commonly choice is the L2 regularizer $\mathcal{R}(\pi) = \|\pi\|^2/2$. 
However, this choice of $\mathcal{R}$ does not satisfy Condition \ref{cond:regularizer}(ii) since $C=\mathcal{R}^*$ will be flat far from the origin.
We suspect that indeed $\mathcal{R}$ needs to be entropy-like, more specifically a variant of Legendre type \cite{rockafeller1997convex} for spaces with empty interior.

\paragraph{Wasted effort from strategizing}
Theorem \ref{thm:ftrl-accuracy} shows that Multiplicative Weights is always $4\eta$-approximately truthful, and is additionally $\epsilon$-optimal with $O(\log(n)/\eta\epsilon)$ events when one chooses $\eta \leq \epsilon/40$.
While the event complexity is optimized by taking $\eta=\epsilon/40$, one important reason to choose $\eta$ even smaller is to control the cost of strategizing by experts.
As our mechanism is not exactly truthful, experts may waste effort modeling their competitors and computing best responses, instead of spending that effort on improving their predictions~\cite{kaggle2017march}.
Choosing $\eta$ even smaller would decrease the benefit of strategizing, at the cost of increasing the event complexity of the mechanism. 
An interesting future direction is to study this tradeoff both theoretically and empirically, in particular to give guidance as to what setting of $\eta$ would eliminate strategizing entirely in practice.

\paragraph{Exact truthfulness}
Our analysis of ELF shows that, though it is exactly truthful, its event complexity is limited to $\Theta(n \log n)$ events and any similar point-per-round mechanisms cannot do better while remaining truthful.
(As an aside, one future direction is to strengthen this bound to $\Theta(n \log (n) / \epsilon^2)$.)
Meanwhile, we gave an approximately truthful mechanism which achieves the optimal event complexity of $O(\log (n) / \epsilon^2)$.
An interesting open question remains as to what happens in the gap between these two bounds. 
Are there exactly truthful mechanisms with optimal event complexity?
One approach could be to further control the curvature of $U_i$ to show concavity jointly in all reports, i.e., with respect to the vector $r_i$.
Then fixed point theorems would give us an equilibrium, and the revelation principle a truthful mechanism (for the solution concept of Nash equilibria).
One challenge is, naively, the bound we achieve from eq.~\eqref{eq:hessian-U} picks up a factor of $m$ because the norm of $\nabla q$ could be order $m$.
Still, a tighter analysis of the curvature may suffice.

\paragraph{Extensive-form strategies in online learning}
Our formulation of the online incentive-compatible learning problem includes a rational strategic game setting as a special case.
That special case can be viewed as a simultaneous-move game rather than a sequential one: each expert decides on a response $r_i = (r_{it})_{t=1}^T$ to a fixed plan of reports $R_{-i}$ of the opponents.
A nice extension would be to show the same approximate-truthfulness guarantee while expanding the strategy set to allow for contingent plans, i.e. reports at time $t$ that depend on the opponents' actions prior to $t$.
We chose to avoid this approach due to the complexity of a model that captures both contingent strategies and possibly-time-inconsistent preferences, e.g. myopic experts.
We conjecture that our approximate truthfulness results for FTRL would extend to this formalization as well, however, thanks to the robustness of the dominated-strategies approach.

\paragraph{Other scoring rules}
Most of our results likely extend to scoring rules other than the quadratic score.
It seems the principal requirements of the scoring rule, aside from being proper~\cite{gneiting2007strictly}, are strong concavity and a bounded derivative.
Less clear is to what extent our results hold when moving beyond binary outcomes, and the correct dependence on the number of possible outcomes in our bounds.

\subsection*{Acknowledgements}
We thank Jens Witkowski
and
Rupert Freeman
for ideas, suggestions, and detailed feedback,
and to Chara Podimata, David Parkes, and Eric Neyman for comments.
This material is based upon work supported by the National Science Foundation under Grant IIS-2045347.

\bibliographystyle{ACM-Reference-Format}
\bibliography{references}

\break
\appendix
\section{ELF Sample Complexity} \label{sec:appendix-elf-bounds}
In this section, we prove the two main results of \S~\ref{sec:elf}, an upper and lower bound on the sample complexity $m^*$ of ELF.

\subsection{Upper Bound} \label{sec:appendix-elf-upper}

We first need some intermediate results.

\begin{lemma} \label{lemma:elf-compare}
  If $a_i > a_j + \epsilon$, then in the ELF mechanism, $\E[F_i] - \E[F_j] > \tfrac{m \epsilon}{n-1}$.
\end{lemma}
\begin{proof}
  Using the definition of $F_i$ and Lemma \ref{lemma:quad-accuracy}, we have
  \begin{align*}
  \E(F_i) - \E(F_j) &= \E\left( \sum_t \frac{1}{n}\left(S(r_{it}, y_t) - \frac{S(r_{jt}, y_t)}{n-1} - S(r_{jt}, y_t) + \frac{S(r_{it}, y_t)}{n-1}\right) \right) \\
  &= \frac{n}{n-1} \E\left( \frac{1}{n}\sum_t\left(S(r_{it}, y_t) - S(r_{jt}, y_t)\right) \right) \\
  &= \frac{1}{n-1} \E\left( m a_i - m a_j \right) \\
  &> \frac{m \epsilon}{n-1}  ~.
  \end{align*}
\end{proof}

Therefore, if forecaster $j$ is not $\epsilon$-optimal, $F_j$ will be at least $\frac{m \epsilon}{n-1}$ worse than the most accurate forecaster's score in expectation.

\begin{theorem}[Bernstein's Inequality \cite{shalev2014understanding}]
  \label{bernstein}
  Given independent random variables $X_i$ for $1 \leq i \leq n$ such that $0 \leq X_i \leq 1$ almost surely, let $Z = \sum_i X_i$. Then,
  \begin{align*}
     \Pr\left[Z - \E[Z] > b \right] <  \exp\left(\frac{-b^2 }{2\left(\Var(Z) + \frac{b}{3}\right)} \right) ~,
     \Pr\left[\E[Z] - Z > b \right] <  \exp\left(\frac{-b^2 }{2\left(\Var(Z) + \frac{b}{3}\right)} \right) ~.
  \end{align*}
\end{theorem}

Now, we are ready to prove Theorem \ref{thm:elf-upper-bound}.

\begin{proof} Note that for any $t$, the probability of $i$ winning the point, $f_{it}$ (Equation \ref{eqn:wagering}), satisfies $0 \leq f_{it} \leq \frac{2}{n}$, because quadratic scores are in $[0,1]$.
With $F_{it} \sim \text{Bernoulli}(f_{it})$, we have for $n \geq 3$,
\begin{align*}
  \Var(F_{it})
    &= f_{it}(1-f_{it}) \\
    &\leq \frac{2}{n}\left(1- \frac{2}{n}\right) \\
    &\leq \frac{2}{n}.
\end{align*}
By independence, $\Var(F_i) \leq \sum_{t=1}^m \Var(F_{it}) \leq \tfrac{2m}{n}$.

Let $i$ be the best forecaster.
By Lemma \ref{lemma:elf-compare}, for any $j$ with $a_j < a_i - \epsilon$, we have $\E[F_i] - \E[F_j] \geq \tfrac{m \epsilon}{n-1}$.
For the mechanism to fail to select an $\epsilon$-optimal forecaster, at least one of the following must happen:
\begin{itemize}
    \item $F_i < \E[F_i] - \tfrac{m\epsilon}{2(n-1)}$, or
    \item there exists $j \neq i$ with $F_j > \E[F_j] + \tfrac{m\epsilon}{2(n-1)}$.
\end{itemize}
By Bernstein's inequality, the probability of the first event is at most
\begin{align*}
    \exp\left(\frac{-\left(\frac{m\epsilon}{2(n-1)}\right)^2}{2\left(\frac{2m}{n} + \frac{m\epsilon}{6(n-1)}\right)}\right) 
    &\leq \exp\left(\frac{-\left(\frac{m\epsilon}{2(n-1)}\right)^2}{2\left(\frac{12m + m\epsilon}{6(n-1)}\right)}\right)
      & \text{using } \frac{2m}{n} \leq \frac{2m}{n-1} \\
    &\leq \exp\left(\frac{-\left(\frac{m\epsilon}{2(n-1)}\right)^2}{\frac{5m}{n-1}}\right)
      & \text{using } \epsilon \leq 3  \\
    &= \exp\left(\frac{-m \epsilon^2}{20(n-1)}\right)  \\
    &\leq \frac{\delta}{n}
\end{align*}
for $m \geq \frac{20(n-1)\ln(n/\delta)}{\epsilon^2}$.
The same calculation shows that the probability of the second event, for any fixed $j \neq i$, is also bounded by $\tfrac{\delta}{n}$.
A union bound over $i$ and the at-most $n-1$ $\epsilon$-suboptimal forecasters $j$ gives that, except with probability at most $\delta$, none of the suboptimal forecasters is selected.
\end{proof}

\subsection{Lower Bound} \label{sec:appendix-elf-lower}
Now, we show the corresponding lower bound given in Theorem \ref{thm:elf-lower-bound}.

\begin{lemma}
\label{lemma:balls_bins_root}
For any $c \in (0, \frac{1}{8})$, let $x_c$ be the largest root of $f_c(x) = 1 + x(\log c - \log x + 1) - c$.
Then, $x_c > 4c$. 
\end{lemma}

\begin{proof}
By \citet[Lemma 3]{raab1998balls}, $x_c$ exists and $x_c > c$. Now, assume $x_c \leq 4c$. We can rewrite the condition $f_c(x_c) = 0$ as $1 + \frac{1}{x_c} = \frac{c}{x_c} + \log \frac{x_c}{c}$. 

Since $x_c \leq 4c < \frac{1}{2}$, $\frac{1}{x_c} > 2$, and $\log \frac{x_c}{c} \leq \log 4$. Since $x_c > c$, $\frac{c}{x_c} < 1$. Combining these, we have 
\[\frac{c}{x_c} + \log \frac{x_c}{c} < 1 + \log 4 < 3 < 1 + \frac{1}{x_c} ~.\]
This is a contradiction, so $x_c > 4c$.
\end{proof}

\begin{lemma}
\label{lemma:balls_bins_bound}
  Fix any $c \in (0, \frac{1}{8})$.
  Suppose we throw $m = c n \log n$ balls uniformly and independently at random into $n$ bins.
  Let $M$ be the maximum number of balls in any bin.
  Then, $\Pr[M > 4c \log n + 1] = 1 - o(1)$. 
\end{lemma}
\begin{proof}
By \citet[Theorem 1]{raab1998balls}, for any $\alpha \in (0, 1)$, we have $\Pr[M > (x_c + \alpha - 1) \log n] = 1 - o(1)$.
By Lemma \ref{lemma:balls_bins_root}, $x_c > 4c$.
For $n > \exp(\frac{1}{x_c - 4c})$, we have $x_c - 4c - \frac{1}{\log n}$, so we can choose $\alpha$ such that $1 - \alpha < x_c - 4c - \frac{1}{\log n}$. Then, we have $x_c + \alpha - 1 \geq 4c + \frac{1}{\log n}$, so $\Pr[M > 4c \log n + 1] = 1 - o(1)$.
\end{proof}

Theorem \ref{thm:elf-lower-bound} then follows.

\begin{proof}
Pick any $\delta < \frac{1}{2}$, $c \in (0, \frac{1}{8})$ and set $m = 2 c n \log n$. Consider the case where every event is identical with $\theta_t = 1$, $r_{1t} = 1$ and $r_{jt} = 0$ for every $j > 1$. Therefore, forecaster 1 always forecasts correctly and has a quadratic score $S(1, 1) = 1$, while everyone else is wrong and has a quadratic score of $S(0, 1) = 0$. Then, $a_1 = 1$, and $a_j = 0$, so for $\epsilon < 1$, forecaster 1 is the only $\epsilon$-optimal choice.

For every event,
\begin{align*}
  f_{1t} &= \frac{1}{n}\left(1 + S(r_{1t}, 1) - \frac{1}{n-1} \sum_{j > 1} S(r_{jt}, 1) \right) \\
      &= \frac{1}{n}\left(1 + 1 - \frac{1}{n-1} \sum_{j > 1} 0 \right)\\
      &= \frac{2}{n} ~.
\end{align*}
Since all the other forecasters are symmetric, the $f_{jt}$ are all the same. Since they sum to $\frac{n-2}{n}$, we have $f_{jt} = \frac{n-2}{n(n-1)} < \frac{1}{n}$. 

We can reframe the way winners are chosen for each event. Instead of holding a lottery, we first flip a coin that gives forecaster 1 a win with probability $\frac{2}{n}$. If it does not give them the win, we run a normal lottery for the remaining forecasters, where they each have probability $\frac{1}{n-1}$ of winning. This gives them an overall probability of $\frac{n-2}{n(n-1)}$ to win each event, so it is the same the original lottery. 

Since all the events are uniform, the expected number of lotteries forecaster 1 will win is $m f_{1t} = \frac{2m}{n} = 4 c \log n$. Specifically, their points will follow the Bernoulli distribution $B(m, \frac{2}{n})$. With some probability $C_1 = \Pr[B(m, \frac{2}{n}) < 4 c \log n + 1] > \frac{1}{2}$ \cite{kaas1980mean}, forecaster 1 will win less events than expected.

The distribution of points for the remaining forecasters will follow a uniform multinomial distribution. Specifically, we can model it by throwing balls into bins.
Conditioned on forecaster 1 winning at most $2c \log n$ points, the remaining $2c (n-2) \log n \geq c n \log n$ points are won by the remaining forecasters. Let $M = \max_{j > 1} F_j$ denote the maximum number of points of any of those forecasters. By Lemma \ref{lemma:balls_bins_bound}, $\Pr[M > 4c \log n + 1] = 1 - o(1) > 2 \delta$ for sufficiently large $n$. 

If forecaster 1 scores less than $4c \log n$ and some other forecaster scores more than $4c \log n$, then forecaster 1 will not be chosen, so $\elf$ will not be $\epsilon$-optimal with probability at least 
\[ \Pr[M > 4c \log n + 1] \Pr[B(m, \frac{2}{n}) < 4 c \log n + 1] > 2 \delta C_1 > \delta\]
Therefore, $m^* > \frac{n}{4} \log n$.
\end{proof}

\subsection{General point-per-round mechanisms}
\label{sec:appendix-general-elf}
In fact, we can generalize this lower bound to "ELF-like"
mechanisms that independently awards a point to a forecaster for each event, and then choose the one with the highest score. We show that if the rule to choose the winner of each event is truthful and \emph{normal}, then it will take the same form as ELF modulo the scoring rule used. 
\begin{definition}
A mechanism that awards a single point per event is \textbf{normal} if any change to player $j$'s reports that increases their chance of winning an event's point does not increase the chance that player $i \neq j$ wins that point. Similarly, any change in report that decreases the chance $j$ wins a point should not decrease any other players chance.
\end{definition}
Essentially, normality means that no player can sabotage some other player's chances by manipulating their own reports.
Although not strictly necessary, it is generally true for good selection mechanisms and has been studied in the context of wagering mechanisms by \citet{lambert2008self}. We also focus on mechanisms that are symmetric for all forecasters.
\begin{definition}
A mechanism is \textbf{anonymous} if swapping the reports of any two forecasters also swaps their win probabilities. 
\end{definition}
Adopting the same notation used for ELF, let $f_{it}$ be the probability that forecaster $i$ receives a point for event $t$.
\begin{definition}
A mechanism is \textbf{ELF-like} if it independently awards 1 point per event to a single forecaster in a normal, anonymous, and truthful way, and selects the forecaster with the most points.
\end{definition}
\begin{definition}
A \textbf{wagering mechanism} \cite{lambert2008self} is a one-shot game where players wager money on predictions, and receive payouts as a function of their performance and wager. Formally, a \textbf{wagering mechanism} for $n$ players is a vector of payout functions $\Pi = (\Pi_i(r, w, y))_{0 \leq i \leq n}$ that are a function of the players' reports $r$, their wagers $w$, and the true outcomes $y$.
\end{definition}
\begin{lemma}
For every event $t$, an ELF-like mechanism satisfies
\[ f_{it} = \frac{1}{n} + g(r_{it}, y_t) - \frac{1}{n-1} \sum_{j \neq i} g(r_{jt}, y_t) \]
  where $g$ is a proper scoring rule whose range is in an interval of length $\tfrac{1}{n}$.
\end{lemma}
\begin{proof}
Note that, by definition, an ELF-like mechanism is a wagering mechanism with fixed, unknown wagers.
Specifically, we have $f_{it} = \Pi_i(r_{it}, w_{it}, y_t)$.
By anonymity, we can swap any pair of reports, and the outputs of the $\Pi_i$ will be permuted swapped accordingly.
Therefore, either the wagers are identical, or the $\Pi_i$ do not depend on them.
In either case, we can let them all identically be any constant $w$.
Note that we can adjust the mechanism to accommodate for any choice of $w$.
Therefore, we choose $w = \frac{1}{n}$. 

The total payout of an ELF-like mechanism on any round $t$ must be 1 because it is just the sum of the probabilities that any player is chosen.
Therefore, since $\sum_i w_{it} = n \frac{1}{n} = 1$, every round of an ELF-like mechanism is budget balanced.
By assumption, each round is also truthful, normal, and anonymous.
By \citet[Lemma 4]{lambert2008self}, we have that
\[ f_{it} = \frac{1}{n} + g(r_{it}, y_t) - \frac{1}{n-1} \sum_{j \neq i} g(r_{jt}, y_t)\]
where $g$ is a proper scoring rule whose range is in an interval of length $\tfrac{1}{n}$.
\end{proof}

\begin{corollary}
Any ELF-like mechanism $M$ has an event complexity of, $m^*(n, \epsilon, \delta)$ is $\Omega(n \log n)$.
\end{corollary}
\begin{proof}
Using the form above, since $g$ is a proper scoring rule on an interval of size at most $\frac{1}{n}$, we have $f_{it} \leq \frac{2}{n}$. By normality, this maximum is achieved when $g(r_{it}, w_t)$ is maximal, so $r_{it} = w_t$, and the $g(r_{jt}, w_t)$ are all minimal. By anonymity, all other forecasters $j \neq i$ must have the same score, $f_{jt} \geq \frac{1}{n-1} \left(1 - \frac{2}{n}\right) = \frac{n-2}{n(n-1)}$. This satisfies the conditions used in the proof of Theorem \ref{thm:elf-lower-bound}, so that analysis can applied as well. Therefore, $m^*(n, \epsilon, \delta) > \frac{n}{4} \log n$.
\end{proof}

\section{Event Complexity Lower Bound} \label{sec:appendix-lower-bound}
Here, we formally prove a lower bound on nonstrategic event complexity, Theorem \ref{thm:technical-lower-bound}, which directly implies the stated lower bound, Theorem \ref{thm:main-lower-bound}.

First, we recall some definitions from PAC learning.
Given a feature space $\X$, a hypothesis class $\H$ is a set of hypotheses $h: \X \to \{0,1\}$.
Given a distribution $D$ on $\X \times \{0,1\}$, the \emph{risk} of $h$ is $L_D(h) := \Pr[h(x) \neq y]$, with probability over $(x,y) \sim D$.
The \emph{excess risk} is $EL_D(h) := L_D(h) - \min_{h^* \in \H} L_D(h^*)$.

We write $D_x$ for the marginal probability $\Pr[y=1 \mid x]$.
In particular, we will focus on distributions whose marginal on $\X$ is uniform; we call these $\X$-uniform distributions.

$S_m$ denotes a list of $m$ samples $(x,y)$ drawn independently from $D$.
A learner $A$ is a function taking a list of samples to a hypothesis, i.e. $A(S_m) \in \H$.
A learner is said to \emph{agnostically $(\epsilon,\delta)$-PAC learn $\H$ with $m$ samples} if for all $D$, $\Pr[EL_D(A(S_m)) \leq \epsilon] \geq 1 - \delta$.
That is, with probability $1-\delta$, the excess risk of the algorithm's output, on $m$ samples, is at most $\epsilon$.

We will need a specialization of the following generic lower bound on samples required for agnostic PAC learning.
\begin{theorem}[Agnostic PAC lower bound, Theorem 3.7 of \cite{mohri2018foundations}] \label{thm:agnostic-pac}
  There exist constants $C_0, \delta_0 > 0$ such that the following is true.
  For any $\X,\H$, $\epsilon > 0$, and learner $A$, if $A$ agnostically $(\epsilon,\delta_0)$-PAC learns $\H$ with $m$ samples, then $m \geq \frac{C_0 \cdot d}{\epsilon^2}$, where $d$ is the VC-dimension of $\H$.
\end{theorem}
We avoid defining VC-dimension because we will only need a simple special case, described next.

\paragraph{Our simple learning setting.}
Given $d \geq 1$, we define $n = 2^d$.
We define $\X_d = \{x_1,\dots,x_d\}$ and we take $\H_d = \{h_1,\dots,h_n\}$ to be the set of all $n$ functions from $\X$ to $\{0,1\}$.
We assert that the VC-dimension of $\H_d$ is indeed $d$; see e.g. \cite[\S~3.3]{mohri2018foundations}).

We require an immediate extension of Theorem \ref{thm:agnostic-pac}, using the fact that for $\X_d$ and $\H_d$, the proof begins by constructing an $\X_d$-uniform distribution.
\begin{theorem}[Immediate extension of Theorem \ref{thm:agnostic-pac}] \label{thm:agnostic-special}
  There exist constants $C_0 > 0$ and $\delta_0 \in (0,\tfrac{1}{8})$ such that the following is true.
  For any $d \geq 1$ and any learner $A$, there exists an $\X_d$-uniform distribution $D$ such that
   \[ \Pr[EL_D(A(S_m)) > \epsilon] \geq \delta_0 \]
  for all
   \[ m \leq \frac{C_0 \cdot d}{\epsilon^2} . \]
\end{theorem}

\paragraph{Reducing to a forecasting setting.}
Given $d \in \mathbb{N}, \epsilon \in (0, \tfrac{1}{2})$, and an $\X_d$-uniform distribution $D$, a forecasting competition setting is \emph{$\epsilon$-good} if it satisfies the following conditions.

There are $n = 2^d$ forecasters.
Each forecaster $j$ is identified with a hypothesis $h_j \in \H_d$.
Choose some $k$ such that $d \geq k \geq d(1-\epsilon)$.
The $m$ events are divided into $k \geq d(1 - \epsilon)$ \emph{groups} of size exactly $m/k$ each.
Each group $i \in \{1,\dots,d\}$ is identified with a distinct point $x_i \in \X_d$.
On all events in group $i$, forecaster $j$ predicts $h_j(x_i)$.
Note these predictions are always extreme, either zero or one.
Finally, the true probability of each event in group $i$ equals  $D_{x_i}$, the marginal probability that $y=1$ given $x=x_i$.

\begin{claim} \label{lemma:eps-good-risk}
  In an $\epsilon$-good forecasting setting, we have for all $j=1,\dots,n$,
  \[ 1 - a_j - 2\epsilon \leq L_D(h_j) \leq 1 - a_j + \epsilon ~. \]
\end{claim}
\begin{proof}
  The key point is that, because forecaster $j$ only predicts extreme values of zero or one, her quadratic loss and zero-one loss are always the same.
  The $O(\epsilon)$ error arises because the number of groups $k$ may not be $d$, but as small as $d(1-\epsilon)$.
  
  Formally, for any hypothesis $h: \X \to \{0,1\}$, it holds that $\Pr[h(x) \neq y \mid x] = D_x (1 - h(x)) + (1 - D_x) h(x)$.
  Next, observe that because $h(x) \in \{0,1\}$, we have $h(x) = h(x)^2$ and $1 - h(x) = (1 - h(x))^2$.
  Write $Q(h(x), D_x) = D_x (1 - h(x))^2 + (1-D(x))h(x)^2$.
  \begin{align*}
    L_D(h_j) &= \Pr[h_j(x) \neq y]  \\
             &= \sum_{i=1}^d \Pr[x_i] \Pr[h_j(x_i) \neq y \mid x_i]  \\
             &= \frac{1}{d} \sum_{i=1}^d D_{x_i} (1 - h_j(x_i)) + (1 - D_{x_i}) h_j(x_i)  \\
             &= \frac{1}{d} \sum_{i=1}^d Q(h_j(x_i), D_{x_i}) .
  \end{align*}
  Now, in an $\epsilon$-good setting, we have $k \geq d(1-\epsilon)$ groups, each of equal size $m/k$.
  Let $S \subseteq \{1,\dots,d\}$ be the set of groups.
  Because $j$'s accuracy is one minus her expected quadratic loss, and using that all groups have the same number of events:
  \begin{align*}
   1 - a_j
     &=     \frac{1}{k} \sum_{i \in S} Q(h_j(x_i), D_{x_i})  \\
     &\leq \frac{1}{k} \sum_{i=1}^d Q(h_j(x_i), D_{x_i}) \\ %
     &\leq \frac{1}{d(1-\epsilon)} \sum_{i=1}^d Q(h_j(x_i), D_{x_i})  \\
     &=    \frac{1}{1-\epsilon} L_D(h_j)  \\
     &\leq \left(1 + 2\epsilon\right) L_D(h_j)  & \text{using $\epsilon \leq \tfrac{1}{2}$} \\
     &\leq L_D(h_j) + 2\epsilon .
  \end{align*}
  Similarly, by observing that $k \leq d$ and $|S| = k \geq (1-\epsilon)d$,
  \begin{align*}
  1 - a_j
    &=    \frac{1}{k} \sum_{i \in S} Q(h_j(x_i), D_{x_i})  \\
    &\geq \frac{1}{d} \sum_{i \in S} Q(h_j(x_i), D_{x_i})  \\
    &=   L_D(h_j) - \frac{1}{d} \sum_{i \not\in S} Q(h_j(x_i), D_{x_i})  \\
    &\geq L_D(h_j) - \frac{1}{d} \left(\epsilon d\right) \left(1 \right)  \\
    &=    L_D(h_j) - \epsilon .
  \end{align*}
\end{proof}

Now, we must argue that when we draw enough samples, we can trim them down to obtain an $\epsilon$-good setting.

\begin{lemma} \label{lemma:good-events-eps}
  Let $0 < \epsilon,\delta< \tfrac{1}{2}$.
  Suppose a set of $m$ samples is drawn uniformly and independently from $\X_d$, for some $m \geq 32 d \ln\left(\frac{1}{\epsilon \cdot \delta}\right)$.
  For all $i$, let $m_i$ be the number of samples of $x_i$.
  Then with probability at least $1 - \delta$,
   \[ \left| \left\{ i : m_i \geq \frac{m}{2d} \right\} \right| \geq (1 - \epsilon)d . \]
\end{lemma}
\begin{proof}
  First, fix any $i \in \{1,\dots,d\}$.
  The number of samples $m_i$ of $x_i$ is distributed Binomial$(m, \tfrac{1}{d})$.
  Therefore, by a Chernoff bound, for any $\gamma \in (0,1)$,
  \begin{align*}
  \Pr\left[ m_i < \frac{m}{d}\left(1 - \gamma\right) \right] &\leq e^{-\gamma^2 m / 2d}  \\
    \implies \Pr\left[ m_i < \frac{m}{2d}\right] &\leq e^{-m / 8d} .
  \end{align*}
  Now, we divide into two cases.
  Case 1 is the usual case, and Case 2 is the case that all $i$ ``succeed'' with high probability.
  
  \textbf{Case 1: $d \geq \frac{2 \ln\tfrac{1}{\delta}}{\epsilon}$.}
  Here, for each $i$, we obtain using only the definition of $m$:
  \[ \Pr\left[m_i < \frac{m}{2d}\right] \leq e^{-m / 8d} \leq e^{-m / 32d} = \epsilon \delta  \leq \frac{\epsilon}{2}. \]
  Let $E_i$ be one if $m_i < \tfrac{m}{2d}$ and zero otherwise.
  By \citet{joagdev1983negative}, the variables $(m_1,\dots,m_n)$, as components of a multinomial distribution, are \emph{negatively associated}, and thus the indicators $E_i$ are as well (they are nonnegative increasing functions of $m_i$).
  Therefore\footnote{Intuitively we have negative association because the more samples we have of $x_j$, the fewer we expect of $x_i$. Negative association implies in particular that $\E[e^{\sum_i X_i}] \leq \prod_i \E[e^{X_i}]$, allowing Chernoff-type upper tail bound proofs to go through unchanged. For more, see e.g. \citet{joagdev1983negative}.} the sum $\sum_{i=1}^d E_i$, with some mean $\mu$, obeys the following Chernoff bound for $c > 1$: 
  \begin{align*}
  \Pr\left[\sum_i E_i \geq c \mu \right]
    &\leq e^{-\frac{(c-1)^2 \mu}{c+1}}  \\
    &\leq e^{-\frac{c \mu}{2}}  & \text{if $c \geq 2$.}  \\
  \end{align*}
  By a change of variables, for any $t \geq 2\mu$, we have $\Pr[\sum_i E_i > t] \leq e^{-t/2}$.
  Above, we showed $\Pr[E_i=1] \leq \tfrac{\epsilon}{2}$, so $\mu \leq \tfrac{\epsilon d}{2}$.
  Therefore, by choosing $t = \epsilon d$, and using our case assumption,
  \begin{align*}
  \Pr[\sum_i E_i \geq \epsilon d]
    &\leq e^{-\epsilon d/2}  \\
    &\leq e^{-\ln(1/\delta)}  \\
    &=    \delta ,
  \end{align*}
  as desired.
  
  \textbf{Case 2: $d \leq \frac{2 \ln\tfrac{1}{\delta}}{\epsilon}$.}
  Here, we will actually show that all $i$ have $\frac{m}{2d}$ samples with high probability.
  Note in this case, $m \geq 32 d \ln\left(\frac{d}{2 \delta \ln\tfrac{1}{\delta}}\right)$.
  By a union bound,
  \begin{align*}
  \Pr\left[ \exists i : m_i < \frac{m}{2d}\right]
    &\leq d e^{-m / 8d}  \\
    &\leq d e^{-4 \ln(d / 2 \delta \ln(1/\delta))}  \\
    &=    d \left(\frac{2 \delta \ln\tfrac{1}{\delta}}{d}\right)^4  \\
    &\leq 16 \delta^4 \left(\ln \tfrac{1}{\delta}\right)^4  \\
    &\leq \delta .
  \end{align*}
  The last step is justified as follows: the right hand side is $\delta f(\delta)$ for $f(\delta) := 16 \delta^3 \left(\ln \tfrac{1}{\delta}\right)^4$.
  We observe that $f(\delta) < 1$ for all $\delta \in [0,1]$, obtaining the result.
  For the observation, we use the finding that $f(\delta)$ is maximized on $[0,1]$ by $\delta^* = e^{-4/3}$, by inspection of the derivative (it is positive below $\delta^*$ and negative above).
\end{proof}

\paragraph{PAC reduction.}
We reduce an instance of PAC learning on $\X_d,\H_d$ with an $\X_d$-uniform distribution and $m$ samples to a forecasting competition mechanism $M$.
We assume $m$ is a multiple of $2d$.
Given the $m$ samples, our learner CompetitionLearn$_M$ acts as follows.

For each $i \in \{1,\dots,d\}$, let $S_i$ be the set of samples $(x,y)$ with $x = x_i$.
If $|S_i| \geq \frac{m}{2d}$, we create a group of events for $x_i$, otherwise, we discard $S_i$.
If created, the group contains exactly $\frac{m}{2d}$ events.
The $n$ forecasters are identified with the $n$ functions in $\H_d$, as described in the $\epsilon$-good setting.
We then simulate the forecasting competition.
The realizations of the events in a group $i$ are the first $\frac{m}{2d}$ labels of samples in $S_i$.
That is, if $S_i = (x_i, y_1), (x_i, y_2), \dots$, then the event realizations are $(y_1,\dots,y_{m/2d})$.
The forecasting competition selects a forecaster $j$.
We then return the corresponding hypothesis $h_j \in \H_d$.

\begin{corollary} \label{cor:comp-learn-eps-good}
  Let $0 < \epsilon,\delta < \frac{1}{2}$.
  On $m \geq 32d \ln\left(\frac{1}{\epsilon \delta}\right)$ samples from an $\X_d$-uniform distribution $D$, the forecasting competition produced by CompetitionLearn is $\epsilon$-good with probability at least $1-\delta$.
\end{corollary}
\begin{proof}
  We verify the conditions of an $\epsilon$-good forecasting competition.
  We have $\epsilon \in (0,\tfrac{1}{2})$, an $\X_d$-uniform $D$, and the $n = 2^d$ forecasters identified with $\H_d$.
  Each group of events has exactly the same size, $\frac{m}{2d}$.
  By Lemma \ref{lemma:good-events-eps}, with probability at least $1-\delta$, there are at least $d(1-\epsilon)$ groups of events.
  Finally, for each event in group $i$, the true outcome is equal to $y_i$, which is indeed drawn independently conditioned on $x_i$ from the marginal distribution Bernoulli$(D_{x_i})$.
\end{proof}
  
\begin{prop} \label{prop:complearn-reduction}
  Suppose there is a forecasting competition $M$ that, given $n = 2^d$ truthful participants, and at least $m$ events, guarantees to select an $\epsilon$-optimal forecaster with probability at least $1-\delta$.
  
  Then CompetitionLearner$_M$ agnostic $(4\epsilon,2\delta)$-PAC learns the class $\H_d$ on uniform-$\X_d$ distributions using $m' = \max\{2m, 32 d \lceil \ln\tfrac{1}{\epsilon \delta} \rceil \}$ samples.
\end{prop}
\begin{proof}
  By Corollary \ref{cor:comp-learn-eps-good}, on input $\epsilon,\delta$, drawing $m'$ samples, with probability at least $1-\delta$, CompetitionLearn$_M$ produces an $\epsilon$-good forecasting setting.
  By assumption, $M$ produces with probability $1-\delta$ a forecaster $j$ satisfying $a_j \geq \max_{j'} a_{j'} - \epsilon$.
  By a union bound, both events occur except with probability at most $2\delta$.
  By Lemma \ref{lemma:eps-good-risk}, $a_j \leq 1 - L_D(h_j) + \epsilon$, and for all $j'$, $a_{j'} \geq 1 - L_D(h_{j'}) - 2\epsilon$.
  Rearranged, we obtain
  \begin{align*}
     1 - L_D(h_j) + \epsilon
     &\geq \max_{j'} \left(1 - L_D(h_{j'}) - 2\epsilon\right) - \epsilon  \\
   \iff L_D(h_j) &\leq \min_{j'} L_D(h_{j'}) + 4\epsilon .
  \end{align*}
\end{proof}

This reduction allows us to finally state our main lower bound.
\begin{theorem} \label{thm:technical-lower-bound}
  There exist constants $C_0 > 0, \delta_0 \in (0,\tfrac{1}{8})$, and $\epsilon_0 > 0$ such that the following holds.
  If $\epsilon < \epsilon_0$ and a forecasting competition on $n$ forecasters selects an $\epsilon$-optimal forecaster with probability at least $1 - \delta_0$, then its number of events satisfies $m \geq \frac{C_0\log(n)}{16 \epsilon^2}$.
\end{theorem}
\begin{proof}
  We suppose without loss of generality that $n = 2^d$ for some $d \in \mathbb{N}$.
  Otherwise, we can round $n$ up to the nearest power of $2$, use the following argument along with some dummy forecasters that are never selected.
  In the bound, $\log(n)$ is replaced with $\log(n/2) \geq \frac{\log(n)}{2}$.
  
  Suppose such a forecasting competition $M$ exists.
  Let $\delta_0,C_0$ be the constants in Theorem \ref{thm:agnostic-special}, the PAC lower bound.
  Set $\delta_1 = \delta_0/2$.
  Set $m'(\epsilon,d) = \max\{2m ~,~ \lceil 32 d \ln \tfrac{1}{\epsilon \delta_1} \rceil \}$.
  By Proposition \ref{prop:complearn-reduction}, CompetitionLearn$_M$ with $m'(\epsilon,d)$ samples on any $\X_d$-uniform distribution satisfies, with probability at least $1-2\delta_1 = 1-\delta_0$, $EL(\text{CompetitionLearn$_M$}(S_{m'}) \leq 4\epsilon$.
  
  By the agnostic PAC lower bound, Theorem \ref{thm:agnostic-special}, this implies that $m'(\epsilon, d) \geq \frac{C_0 d}{16 \epsilon^2}$.
  Now, for all small enough $\epsilon$, $\frac{C_0 d}{16\epsilon^2} > \lceil 32 d \ln \tfrac{1}{\epsilon \delta_0 \rceil}$.
  Therefore, for all small enough $\epsilon$, $m' = 2m$.
  That is, for some $\epsilon_0$ and all $\epsilon < \epsilon_0$, the forecasting competition's number of samples satisfies
    \[ m = \frac{m'}{2} \geq \frac{C_0 d}{16 \epsilon^2} . \]
\end{proof}

\section{Follow-the-Perturbed-Leader Example: Report Noisy Max} \label{sec:laplace}

In this section, we consider an example of the Follow The Perturbed Leader (FTPL) online learning framework.
It is known that many FTPL algorithms can be recast as Follow-the-Regularized-Leader (FTRL) with appropriate choice of regularizer.
However, here, we analyze the \emph{ReportNoisyMax} mechanism directly.
Specifically, we select a winner taking each forecaster's total score, perturbing it by adding independent Laplace noise, and taking the max.
The name of the mechanism comes from differential privacy, where it is a basic building block~\citep{dwork2014algorithmic}.
A side effect of this mechanism is that announcing the winner reveals little information about any particular fixed forecast on any particular fixed round.

The Laplace distribution is a two-sided exponential distribution, i.e. its probability density function with mean $\mu$ and scale parameter $b$ is $f(x) = \frac{1}{2b} \exp\left(\frac{-|x-\mu|}{b}\right)$.
We note that adding Gaussian noise instead of Laplace does \emph{not} seem to be truthful, at least with our proof technique.
Roughly, the problem is that, if $f$ is the density of the Gaussian, then for very large $x$, $f(x)/f(x+1)$ grows to infinity.
This implies that the relative incentive to increase one's variance increases unboundedly as one's expectation moves farther from the threshold needed to win.
For the Laplace distribution, of course, this ratio is constant and controlled by $b$.
Interestingly, regret bounds for FTPL have been shown to be related to the \emph{hazard rate} of the perturbation distribution~\citep{abernethy2015fighting}, a very similar quantity, and Gaussian noise appears to fail because of its unbounded hazard rate.
It remains to be seen if there is a formal connection between our requirement (needed for approximate truthfulness) and that one (needed for regret bounds).

\subsection{Approximate truthfulness}
In this section, we directly show that ReportNoisyMax is approximately truthful.
\begin{theorem} \label{thm:laplace-truth}
  ReportNoisyMax, with parameter $b = \frac{4}{\gamma}$ and any $\gamma \in (0,10$, is $\gamma$-approximately truthful.
\end{theorem}
\begin{proof}
We show that if $|r_{it} - p_{it}| > \gamma$, then $r_i$ is strictly dominated.

In particular, we show that for any belief $p_i$, any reports of other players $R_{-i}$, any reports of $i$ on events other than $t$, realizations of the Laplace noise for all players but $i$, and set of outcomes on other rounds $\vec{y}_{-t}$:
\begin{enumerate}
  \item The maximizer $r_{it}^*$ of $U_{it}$ satisfies $|r_{it}^* - p_{it}| \leq \gamma$.
  \item $U_{it}$ is strictly increasing for $r_{it} < r_{it}^*$ and strictly decreasing for $r_{it} > r_{it}^*$.
\end{enumerate}
Now, let $U_i(r_{it})$ denote $i$'s expected utility over the randomness of all events and Laplace noise, for fixed $p_i$, $R_{-i}$, and $r_i$ on rounds other than $t$.
The above immediately implies that $U_i$ is strictly increasing for $r_{it} < p_{it} - \gamma$ and strictly decreasing for all $r_{it} > p_{it} + \gamma$, because it is an expectation over functions that share this property.

So suppose $r_{it} = p_{it} + \gamma + c$ for some $c > 0$.
Then let $r_{it}' = p_{it} + \gamma + c/2$.
We immediately obtain $U_i(r_{it}') > U_i(r_{it})$.
So modifying $r_i$ to $r_i'$ which is identical on all events except $t$ provides strictly higher utility.
The analogous argument goes through analogously if $r_{it} < p_{it} - \gamma$.
So if $|r_{it} - p_{it}| > \gamma$, then $r_i$ is strictly dominated.

\vskip1em
It remains to show items 1 and 2 above for $U_{it}$, i.e. concavity, differentiability, and location of the maximizer.

For $j \neq i$, let $z_{j,0}$ be the total noisy points of $j$ if $y_t = 0$, and symmetrically for $z_{j,1}$.
Let $Z_0 = \max_{j\neq i} z_{j,0}$ and $Z_1 = \max_{j \neq i} z_{j,1}$.
Observe that $|Z_0 - Z_1| \leq 1$, because each $|z_{j,0} - z_{j,1}| \leq 1$.

Let $W \sim \text{Laplace}(b)$ and let $S'$ be the total quadratic score of $i$ summed over events $t' \neq t$.
Then $i$ wins if (ignoring measure-zero ties):
\begin{align*}
  S' + S(r_{it},0) + W &> Z_0  & y_t=0  \\
  S' + S(r_{it},1) + W &> Z_1  & y_t=1 .
\end{align*}
In other words, let $F_0$ be the CDF of a Laplace variable with mean $\mu_0 := Z_0 - S'$ and parameter $b$, and let $F_1$ be the CDF of a Laplace with mean $\mu_1 := Z_1 - S'$ and parameter $b$.
The probability that $i$ wins conditioned on $y_t=0$ is $F_0(S(p_i,0))$, and the probability $i$ wins conditioned on $y_t=1$ is $F_1(S(p_i,1))$.

The probability $i$ wins, from her perspective, is therefore
  \[ U_{it}(r_{it}) := (1 - p_{it}) F_0(S(r_{it},0)) + p_{it} F_1(S(r_{it},1)) . \]

We now claim that the maximizer $r_{it}^*$ of $U_{it}$ satisfies the equation
   \[ r_{it}^* = \frac{p_{it}}{E - p_{it} E + p_{it}} , \]
  where $E = \exp\left(\frac{|S(r_{it},1) - \mu_1| - |S(r_{it},0) - \mu_0|}{b}\right)$.
  Furthermore, we claim $U_{it}$ is strictly increasing for $r_{it} < r_{it}^*$ and strictly decreasing for $r_{it} > r_{it}^*$.
  
  By definition of the Laplace mechanism, the density function is $\frac{dF_0}{dx} = \frac{1}{2b}e^{|x-\mu_0|/b}$.
  Recall $S(r_{it}, y_t) = 1 - (y_y - r_{it})^2$.
  So $\frac{dS}{dr_{it}} = 2(y_t - r_{it})$.
  So
  \begin{align*}
    \frac{dU_{it}}{dr_{it}}
      &= (1-p_{it}) \frac{-r_{it}}{b}e^{|S(r_{it},0) - \mu_0|/b} + p_{it} \frac{1 - r_{it}}{b}e^{|S(r_{it},1) - \mu_1|/b}  \\
      &= \frac{1}{b}e^{|S(r_{it},0) - \mu_0|/b} \left[ p_{it} \left(1 - r_{it}\right) - r_{it} \left(1 - p_{it}\right) E \right]
  \end{align*}
  where $E = \exp\left(\frac{|S(r_{it},1) - \mu_1| - |S(r_{it},0) - \mu_0|}{b}\right)$.
  (Note that $E$ is a function of $r_{it}$, so we will not obtain a closed form solution.)
  
  Setting the derivative to zero, we can rearrange to get
  \begin{align*}
    r_{it} E - r_{it} p_{it} E &= p_{it} - r_{it} p_{it}  \\
    \implies r_{it} \left(E - p_{it} E + p_{it}\right) &= p_{it}  \\
    \implies r_{it} &= \frac{p_{it}}{E - p_{it} E + p_{it}} .
  \end{align*}
  We observe that the derivative is negative if the above equality is instead $>$, and is positive if the above equality is instead $<$.
  Next, $g(r_{it}) := r_{it} - \frac{p_{it}}{E - p_{it}E + p_{it}}$ is strictly increasing in $r_{it}$, is nonpositive when $r_{it}=0$, and nonnegative when $r_{it}=1$.
  The latter two follow from the right term being in $[0,1]$, due to $E > 0$.
  Strictly increasing follows because the derivative of the right term is $\frac{-p(1-p)E C}{(E(1-p)+p)^2}$ where $C = \frac{d}{dr_{it}}\frac{1}{b}\left(|S(r_{it},1) - \mu_1| - |S(r_{it},0)-\mu_0|\right)$.
  Because the quadratic score $S$ is $2$-Lipschitz, $C \in [-\tfrac{4}{b}, \tfrac{4}{b}]$.
  We have $|S(p_j,1) - S(p_j,0)| \leq 1$ and, using the above fact that $|Z_1 - Z_0| \leq 1$, we also have $|\mu_1 - \mu_0| \leq 1$.
  Therefore, $E \in [e^{-2/b}, e^{2/b}]$.
  So
  \begin{align*}
    \frac{dg}{dr_{it}} &\geq 1 - \frac{4Ep(1-p)}{b(E(1-p)+p)^2}   \\
      &\geq 1 - \frac{4p(1-p)E}{b(E^2(1-p)^2 + 2Ep(1-p) + p^2)}  \\
      &\geq 1 - \frac{4Ep(1-p)}{b(2Ep(1-p)}  \\
      &\geq 1 - \frac{2}{b}  \\
      &\geq \frac{1}{2}
  \end{align*}
  for $b \geq 4$.
  So $g$ is strictly increasing, so there is exactly one solution $r_{it}^*$, and the derivative is positive below it and negative above.

\vskip1em
We are ready to complete the proof.
If we replace $E$ with something larger, note $r_{it}^*$ only decreases, and vice versa.
So:
\begin{align*}
  r_{it} &\leq \frac{p_{it}}{e^{-2/b} + p_{it}\left(1 - e^{-2/b}\right)}  \\
      &\leq \frac{p_{it}}{e^{-2/b}}   & \text{$p_{it} \geq 0$ and $e^{-2/b} < 1$} \\
      &\leq \frac{p_{it}}{1 - \tfrac{2}{b}}  & \text{$e^x \geq 1 + x$}  \\
      &=    p_{it} + \frac{2}{b - 2}  \\
      &\leq p_{it} + \frac{4}{b}  & \text{for $b \geq 4$, $b-2 \geq \tfrac{b}{2}$}
\end{align*}
For $b \geq \tfrac{4}{\gamma}$, we get $r_{it} \leq p_{it} + \gamma$.

Next,
\begin{align*}
  r_{it} &\geq \frac{p_{it}}{e^{2/b} - p_{it}(e^{2/b} - 1)}  \\
      &\geq \frac{p_{it}}{e^{2/b}}   & \text{$p_{it} \geq 0$ and $e^{2/b} > 1$}  \\
      &=     p_{it} e^{-2/b}  \\
      &\geq  p_{it} \left(1 - \tfrac{2}{b}\right)  \\
      &\geq  p_{it} - \tfrac{2}{b} .
\end{align*}
For $b \geq \tfrac{2}{\gamma}$, we get $r_{it} \geq p_{it} - \gamma$.
\end{proof}

\subsection{Event complexity}

In this section, define:
\begin{itemize}
  \item $\empiricalquadscore_i = $ $i$'s average quadratic score over the events.
  \item $Z_i$ is $i$'s final noisy score, averaged: $Z_i = \empiricalquadscore_i + \tfrac{1}{m}\text{Laplace}(b)$.
  \item $\expectedquadscore_i = \E[ \empiricalquadscore_i ]$, $i$'s expected average quadratic score.
  \item $\expectedtruescore_i = $ the expected average quadratic score if $i$ had reported truthfully.
\end{itemize}

\begin{lemma} \label{lemma:lap-conc}
  With probability at least $1-\tfrac{\delta}{2}$: For all $n$ forecasters $j$,  $|Z_j - \empiricalquadscore_j| \leq \frac{4}{m \gamma} \ln \left(\frac{2n}{\delta}\right)$.
\end{lemma}
\begin{proof}
  By definition of the CDF, a Laplace$(b)$ variable $W$ has, with probability $1 - \delta'$, $|W| \leq b \ln\tfrac{1}{\delta'}$.
  Set $b = \frac{4}{\gamma}$.
  Set $\delta' = \tfrac{\delta}{2n}$ and union-bound over the $n$ forecasters.
  Then divide by the number of events $m$ to normalize.
\end{proof}

\begin{theorem} \label{thm:laplace-events}
  ReportNoisyMax, with any choice of $\gamma \leq \frac{\epsilon}{14}$ and with Laplace parameter $b = \frac{4}{\gamma}$, has event complexity at most
   \[ m^* \leq \frac{28 \ln(2n / \delta)}{\epsilon \cdot \gamma} . \]
  In particular, $O(\ln(n/\delta) / \epsilon^2)$ events suffices.
\end{theorem}
\begin{proof}
  Let $i$ be the best forecaster.
  By the truthfulness guarantee, in equilibrium, all of $j$'s reports are within $\gamma$ of her beliefs for all $j$.

  Ensure $\gamma \leq \frac{\epsilon}{14}$.
  By $2$-Lipschitzness of the quadratic score, this implies each $j$'s expected quadratic score, compared to if she were truthful, satisfies $|\expectedquadscore_j - \expectedtruescore_j| \leq 2 \gamma \leq \frac{\epsilon}{7}$.
  
  Next: ensure $m \geq \frac{28 \ln(2n/\delta)}{\epsilon \gamma}$.
  The guarantee of Lemma \ref{lemma:lap-conc} gives, with probability $1-\tfrac{\delta}{2}$, for all $j$, we have $|Z_j - \empiricalquadscore_j | \leq \frac{4}{m\gamma} \ln(\frac{2n}{\delta}) \leq \frac{\epsilon}{7}$.
  And because $m \geq \frac{49 \ln(2n/\delta)}{2 \epsilon^2}$, the Hoeffding bound of Lemma \ref{lemma:hoeff-conc} and a union bound gives, with probability $1 - \tfrac{\delta}{2}$, $\expectedquadscore_i - \empiricalquadscore_i \leq \frac{\epsilon}{7}$ and, for all $j \neq i$, $\empiricalquadscore_j - \expectedquadscore_j \leq \frac{\epsilon}{7}$.

  So with probability at least $1-\delta$, by a union bound, the guarantees of Lemma \ref{lemma:hoeff-conc} and Lemma \ref{lemma:lap-conc} both hold.
  Then for all $j$ with $a_j < a_i - \epsilon$,
  \begin{align*}
    Z_j &= \expectedtruescore_j + (\expectedquadscore_j - \expectedtruescore_j) + (\empiricalquadscore_j - \expectedquadscore_j) + (Z_j - \empiricalquadscore_j)  \\
      &\leq \expectedtruescore_j + \frac{3\epsilon}{7}  \\
    &\leq \expectedtruescore_i - \frac{4\epsilon}{7}  \\
    &=    \expectedtruescore_i - \frac{\epsilon}{7} - \frac{\epsilon}{7} - \frac{\epsilon}{7} - \frac{\epsilon}{7}  \\
    &\leq \expectedtruescore_i - \frac{\epsilon}{7} + (\expectedquadscore_i - \expectedtruescore_i) + (\empiricalquadscore_i - \expectedquadscore_i) + (Z_i - \empiricalquadscore_i)  \\
    &= Z_i - \frac{\epsilon}{7}  \\
    &< Z_i .
  \end{align*}
  So no non-$\epsilon$-optimal forecaster is selected.
\end{proof}

\section{Approximate Truthfulness for the Online Learning Setting} \label{sec:appendix-online-approx-truth}

This section proves the approximate truthfulness of Follow the Regularized Leader (FTRL), under some conditions.
We rely on the results developed in Section \ref{subsec:ftrl-approx-truth} and focus on how they adapt to the online learning setting.

\begin{proof}[Proof of Lemma \ref{lemma:online-approx-truth}]
  Fix reports $R_{-i}$ and consider a round $t$.
  We must show that if $r_{it} = r_{it}^t$ satisfies $|r_{it}^t - p_{it}| > \gamma$, it is dominated.
  To capture the sequential nature of the online learning setting, let $\vec y_{-t}^s$ be all outcomes from $1$ to $s-1$ except $t$, and similarly for $p_{i,-t}^s$, and let $U_{it}^s$ be as in eq.~\eqref{eq:bar-U} but with respect to reports $R_{1..s-1}$.
  On round $t$, expert $i$ seeks to maximize the following, in expectation over $\vec{y}_{t..T}$, fixing $\vec{y}_{1.t-1}$:
  \begin{align*}
    U_{it}
    &= \E \sum_{s=t+1}^T \uco_{it}^s M(R_{1..s-1},\vec y_{1..s-1})_i
    \\
    &= \sum_{s=t+1}^T \uco_{it}^s \E M(R_{1..s-1},\vec y_{1..s-1})_i
    \\
    &= \sum_{s=t+1}^T \uco_{it}^s \E U_{it}^s(r_{it}|\vec y_{-t}^s)~.
  \end{align*}
  We are therefore in the same situation as Theorem~\ref{thm:approx-truthful}: as a function of $r_{it}$, the utility $u_{it}$ is a nonnegative linear combination of strictly concave functions, each of which are optimized by a report within $\gamma$ of $p_{it}$, where $\gamma = (\beta+1)\eta$.
  Precisely the same proof gives that in particular if $|\hat r_{it}^t - p_{it}| > \gamma$, then it is strictly dominated.
\end{proof}

\end{document}